\documentclass{article}

\PassOptionsToPackage{numbers, compress}{natbib}
\usepackage[main, final]{neurips_2026}
\usepackage[utf8]{inputenc} % allow utf-8 input
\usepackage[T1]{fontenc}    % use 8-bit T1 fonts
\usepackage{hyperref}       % hyperlinks
\usepackage{url}            % simple URL typesetting
\usepackage{microtype}
\usepackage{subcaption}
\usepackage{booktabs}       % professional-quality tables
\usepackage{amsfonts}       % blackboard math symbols
\usepackage{nicefrac}       % compact symbols for 1/2, etc.
\usepackage{microtype}      % microtypography
\usepackage{xcolor}         % colors
\usepackage{graphicx}
\usepackage{wrapfig}
\usepackage{titletoc}

\usepackage{amsmath}
\usepackage{amssymb}
\usepackage{mathtools}
\usepackage{amsthm}
\definecolor{academicBlue}{RGB}{76,114,176}
\definecolor{academicRed}{RGB}{221,132,82}

\usepackage{algorithm}
\usepackage{algorithmic}
\usepackage{pgfplots}
\pgfplotsset{compat=1.18}
\theoremstyle{plain}
\newtheorem{theorem}{Theorem}[section]

\newtheorem{lemma}[theorem]{Lemma}

\theoremstyle{definition}
\newtheorem{definition}[theorem]{Definition}
\newtheorem{assumption}[theorem]{Assumption}
\theoremstyle{plain}
\newtheorem{remark}{Remark}

\title{HyCO: A Hybrid Neural Solver
for \\ Combinatorial Optimization}

\author{%
  Yuheng Li\thanks{Equal contribution.}\\
  Department of Data Science\\
  College of William \& Mary\\
  \texttt{yli95@wm.edu} \\
  \And
  Di Yang\footnotemark[1] \\
  Department of Data Science\\
  College of William \& Mary\\
  \texttt{dyang06@wm.edu} \\
  \AND
  Haipeng Chen\thanks{Corresponding authors.} \\
  Department of Data Science\\
  College of William \& Mary\\
  \texttt{hchen23@wm.edu} \\
  \And
  Yanhai Xiong\footnotemark[2] \\
  Department of Data Science\\
  College of William \& Mary\\
  \texttt{yxiong05@wm.edu} \\
}

\begin{document}

\maketitle

\begin{abstract}
  Sequential reinforcement learning (RL) solvers and global diffusion model (DM) solvers for neural combinatorial optimization exhibit complementary failure modes under an optimization-regret view. The former enjoys small marginal regret in the early construction stage, but suffers from horizon-wise compounding errors with super-linear regret growth; the latter avoids horizon compounding but incurs linear or sublinear regret w.r.t. the dimension of the remaining unsolved subspace. We propose \underline{Hy}brid Neural Solver for \underline{C}ombinatorial \underline{O}ptimization (\textbf{HyCO}), a hybrid inference algorithm that constructs a solution prefix with an RL solver and adaptively switches to a conditional DM to complete the remaining decisions. To characterize why such hybridization helps, when to trigger the handover, and how to realize it in practice, we first develop a unified error-scaling theoretical framework and prove that, under explicit error-scaling assumptions, i) the hybrid structure achieves strictly lower expected regret than either backbone alone, and ii) there exists a unique optimal trigger step that minimizes the hybrid regret. We then design a lightweight adaptive trigger that combines policy entropy and RL–DM disagreement to detect trajectory-level signals of the regime shift as a practical proxy, since the optimal trigger step is defined at the expected-regret level and is not directly computable on individual trajectories. Experimental results on diverse benchmarks demonstrate that HyCO achieves consistent improvements over both backbones and support the empirical effectiveness of adaptive triggering. Our code and checkpoints are available at \url{https://github.com/wmd3i/HyCO}.

\end{abstract}

\section{Introduction}
Combinatorial optimization (CO) problems, such as the Traveling Salesman Problem (TSP), are fundamental to operations research but notoriously difficult to solve at scale due to their NP-hard nature \citep{Garey1979, Papadimitriou1982}. While traditional exact solvers like Concorde \citep{Applegate2006} provide optimality, their exponential complexity limits scalability. This has spurred the design of neural solvers that learn heuristics directly from data. Among recent neural approaches, two representative solver classes have emerged: sequential construction solvers based on reinforcement learning (RL) \citep{Bello2016, Kool2019,mazyavkina2021reinforcement}, and global generative solvers based on diffusion models (DMs) \citep{Graikos2022, Sun2023}.

Both classes aim to produce high-quality solutions, but they rely on fundamentally different inference mechanisms and exhibit \textbf{complementary error trends} under an optimization-regret view:

% Both classes aim to produce high-quality solutions, but their fundamentally different inference mechanisms lead to \textbf{complementary error trends} under an optimization-regret view:

RL-based sequential solvers typically construct solutions in an autoregressive manner. At early decision steps, the marginal regret is typically small as the solver operates within the high-density regions of the ground-truth state distribution. However, such solvers are inherently vulnerable to \textbf{horizon-wise compounding errors} \citep{ross2010efficient,ross11a2011online}: as the decision horizon extends, even small local deviations steer the state trajectory away from the training state distribution, causing decision errors to propagate and accumulate super-linearly w.r.t. the horizon \citep{jin2018Qlearning,jin2020provably,velegkas2022reinforcement}.

\begin{figure*}[t]
    \vspace{-1em}
    \centering
    \begin{minipage}[c]{0.36\linewidth} 
        \centering
        \includegraphics[height=4.5cm, keepaspectratio]{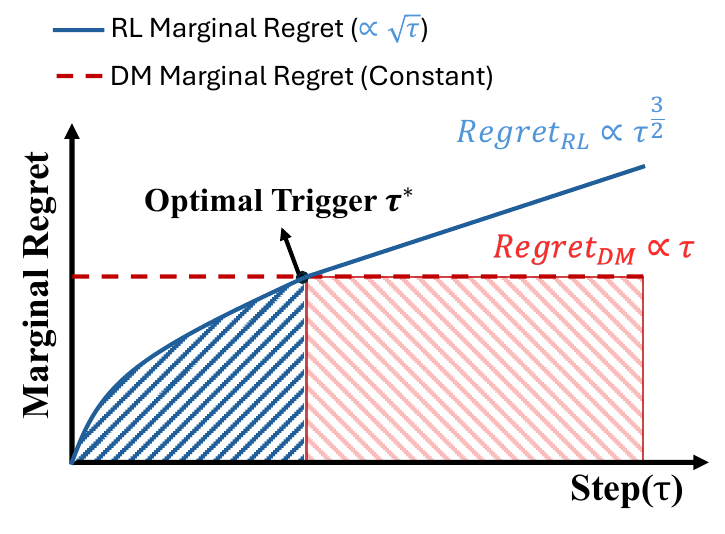} % 确保文件名正确
        \vspace{2pt}
        \centerline{\small (a) Regret Dynamics (Motivation)}
        \vspace{1pt}
    \end{minipage}
    \hfill
    \begin{minipage}[c]{0.56\linewidth} 
        \centering
        \includegraphics[height=4.5cm, keepaspectratio]{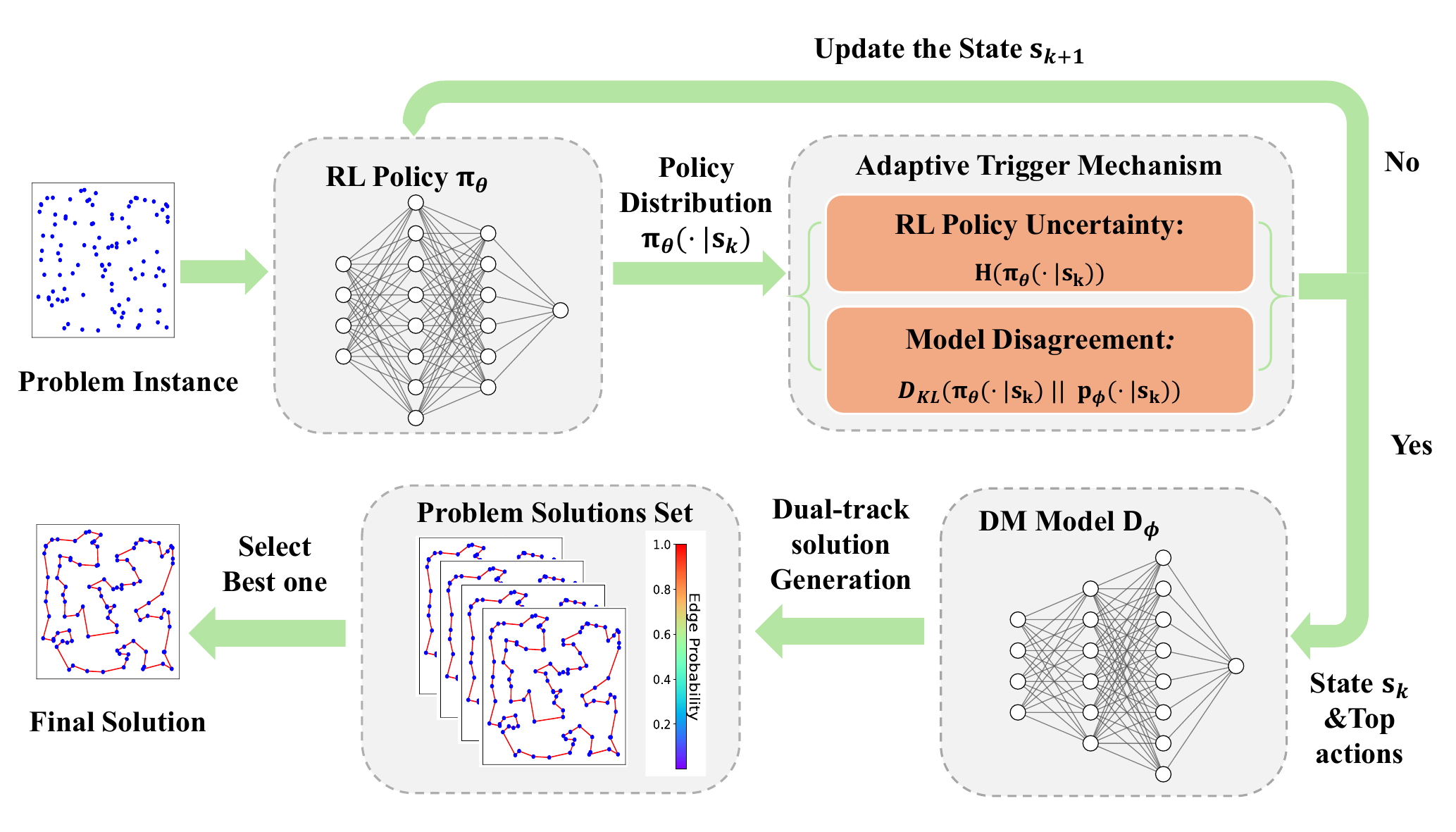} % 确保文件名正确
        \vspace{2pt}
        \centerline{\small (b) HyCO Inference Framework}
    \end{minipage}
    \vspace{1pt}
    \caption{
        \textbf{Motivation and Method.} 
        \textbf{(a)} Schematic illustration of the unified error dynamics. The marginal regret of RL increases as the trajectory lengthens, while that of a conditional DM is constant as the solution prefix grows. Their intersection defines the optimal trigger step $\tau^*$. 
        \textbf{(b)} The HyCO framework, which constructs a prefix with a RL solver, detects the regime shift via a lightweight adaptive trigger, and hands over to a conditional DM for dual-track completion.
    }
    \label{fig:teaser}
    \vspace{-1em}
\end{figure*}

In contrast, DM-based global solvers exhibit a different error profile. Unlike RL solvers, which compound errors over sequential decision steps, DMs incur a solution-space-governed regret that can scale \textbf{linearly} (or sub-linearly) with the effective dimension of the solution space \citep{benton2023nearly}. While generating a full high-dimensional solution from scratch incurs high regret due to the vast search space, the marginal regret of a conditional DM remains constant or may even decrease as the solution prefix grows. As the constraints from the prefix tighten, the remaining subspace shrinks, simplifying the generative task. This inverse relationship implies a structural crossing point where the optimal strategy shifts from RL's construction to DM's generation, as visualized in Figure~\ref{fig:teaser}(a). 

This crossing point motivates a principled hybrid integration: leveraging an RL solver to construct a partial solution prefix and subsequently handing over the remaining decisions to a conditional DM for global generation. Three critical questions arise from this hybrid structure: \textit{why} it is beneficial, \textit{when} to trigger the handover, and \textit{how} to implement it on individual trajectories without oracle access. 

For the first two, we develop a unified error-scaling framework that aligns the regret dynamics of both paradigms under a common metric. The analysis shows that, under explicit assumptions, a single RL-to-DM handover can achieve strictly lower expected regret than either backbone alone, and that the regret-minimizing trigger is unique, occurring at the marginal-regret balance between the two solvers. For the third, since the optimal trigger is defined at the expected-regret level and cannot be directly evaluated along a single trajectory, we propose HyCO (Figure 1(b)), a plug-and-play hybrid inference algorithm. HyCO combines a single-trigger pipeline with dual-track solution generation and employs a lightweight adaptive trigger based on two complementary proxies: policy entropy, which captures local uncertainty of the RL policy, and RL–DM disagreement, which detects a mismatch between RL's local preference and the DM-induced global prior.

Our contributions are summarized as follows: 1) A unified marginal-regret framework proving that a single RL-to-DM handover achieves strictly lower expected regret than either backbone alone, with a unique regret-minimizing trigger. 2) HyCO, a hybrid inference algorithm with an adaptive dual-signal trigger for detecting the regime shift at the trajectory level. 3) Experiments on TSP, Maximum Independent Set (MIS), and Orienteering Problem (OP) show that HyCO consistently improves over both backbones, while oracle validation and ablations support the effectiveness of the adaptive trigger.

\section{Related Work}
\noindent\textbf{RL Solvers.} RL solvers frame CO problems as sequential decision-making. Following the seminal Pointer Network \citep{ Vinyals2015,Bello2016}, \cite{Kool2019} established the Transformer-based blueprint, later enhanced by techniques for multiple optima \citep{Kwon2020}, solution symmetries \citep{Kim2022}. NAR4TSP \citep{xiao2024nar4tsp} proposed a non-autoregressive RL algorithm, but with lower solution quality than the autoregressive versions. Subsequent research expanded RL to graph-based generalization \citep{bengio2020machine, chen2021contingency} and advanced learning paradigms including latent space search \citep{chalumeau2023co}, hierarchical RL \citep{Feng2025hrl4sco}, preference optimization \citep{pan2025preference}, and game-theoretic modeling \citep{li2025ccdorl}. To enhance performance, recent works integrate RL with heavy decoders \citep{luo2023neural} or meta-heuristics like genetic algorithms \citep{kim2025neural} and ant colony sampling \citep{kim2025ant}. For large-scale COPs, specialized architectures such as Invit \citep{fang2024invit}, UDC \citep{zheng2024udc}, and BQ-NCO \citep{drakulic2023bq} have been proposed. Despite their efficiency, these solvers inherently suffer from \textit{horizon-wise compounding errors} due to their myopic construction, leading to super-linear regret growth as the problem scale increases.

\noindent\textbf{Heatmap and Diffusion Models for CO.} 
To overcome the sequential limitations of RL solvers, another line of research has focused on non-autoregressive, global DM methods that generate a solution in a single shot, often by producing a ``heatmap" of edge or node probabilities \citep{LiCK18,FuQZ21}. \cite{Joshi2019} used Graph Convolutional Networks (GCNs) to predict edge inclusion probabilities for TSP. More recently, diffusion models have been adapted for CO. \cite{Graikos2022} mapped TSP instances to images to be solved by a standard image diffusion model. A more direct approach, DIFUSCO \citep{Sun2023}, introduced a graph-based diffusion framework that operates on the problem's native graph structure, casting it as a discrete vector generation task. This paradigm has been explored for TSP and extended to other routing problems like VRP \citep{macoexpander}, while related score-based generative models have been proposed for a broader class of CO problems. While global solvers avoid horizon compounding by capturing global dependencies, they incur a subspace-governed regret that scales with the full solution space dimensionality \citep{benton2023nearly}. Consequently, generating solutions from scratch in high-dimensional settings leads to significant marginal regret, limiting precision unless the effective dimensionality is reduced.

\section{Preliminaries}
%We first make a quick overview of the sequential RL solvers and the global DM solvers.

\subsection{Sequential Inference with RL Solvers}

RL-based solvers formulate CO problems as sequential decision-making processes, where a solution is constructed step by step according to a learned policy. Let $s_k \in \mathcal{S}$ denote the state at step $k$, and let $\pi(\cdot | s_k)$ be an RL policy that outputs a distribution over feasible actions $\mathcal{A}_k$. At each step, an action is selected according to the policy, for example via greedy decoding:
\begin{equation}
a_k = \arg\max\nolimits_{a \in \mathcal{A}_k} \pi(a \mid s_k).
\end{equation}
This sequential construction enables highly efficient inference, requiring only a linear number of forward passes with respect to the solution length. However, decisions made at early steps constrain all subsequent choices. As a result, local errors introduced during the construction process propagate through the remaining horizon, leading to compounding errors and potentially suboptimal solutions.

\subsection{Global Inference with Diffusion Models}

A DM is a non-autoregressive, generative model that learns to produce a complete solution holistically. For CO problems like TSP, where a solution can be represented by a binary adjacency matrix $A \in \{0, 1\}^{N \times N}$, the model uses a discrete diffusion process. It consists of a denoising network, $D_\phi$, trained to reverse a noising process that gradually corrupts the solution by flipping its binary entries. Notably, this process can be conditioned on prior information, such as a given prefix of a tour in a TSP, to guide the generation.

Inference (sampling) starts with a matrix of pure random noise, $A_T$, and iteratively refines it over $T$ total steps to produce a clean solution, $x_0$. At each denoising step $t$ (different from the step $k$ in the episode for RL), the denoising network $D_\phi$ predicts the original clean solution $\hat{x}_0$ based on the current noisy matrix $A_t$ and a conditioning prefix $c$. This prediction is then used to sample a slightly less noisy matrix $A_{t-1}$. The reverse step is generally formulated as:
\begin{equation}
A_{t-1} \sim p_\phi(A_{t-1}|A_t, c), t \in \{1, ..., T\}, A_t \in [0, 1]^{N \times N},%\chen{t range}
\end{equation} where $p_\phi$ denotes the learned reverse transition distribution. This iterative process allows the model to capture complex global dependencies, leading to high-quality results. 

\section{Unified Regret Framework and Optimal Trigger Analysis}\label{sec:theory_structure}
% To synergize sequential RL and global DM, we first align their disparate inference paradigms under a unified regret metric. We establish a rigorous error-scaling analysis to prove the existence and uniqueness of an optimal timing of triggering the DM solver. Guided by these theoretical insights, we then introduce the HyCO algorithm in Section~\ref{sec:hyco}, a plug-and-play hybrid inference algorithm via a lightweight adaptive trigger. \yuheng{polish again. Add the main contribution of theory in this paragraph.}

%\subsection{Regret Dynamics and Optimal Trigger Structure}\label{sec:theory_structure}

This section develops a unified theoretical framework for analyzing our RL-to-DM hybrid structure. We begin by aligning RL and DM under a unified regret metric (Definition~\ref{def:regret}), then characterize their contrasting error-scaling behaviors: the DM's regret scales sublinearly with the unsolved dimension (Lemma~\ref{thm:dm_scaling}), while the RL's regret grows super-linearly with the decision horizon (Lemma~\ref{lemma:rl_scaling}). Building on these lemmas, we prove that under the boundary conditions, there exists a unique optimal trigger step $\tau^*$ at which the handover from RL to DM minimizes the total hybrid regret (Theorem~\ref{thm:optimal_handoff}). This optimal step satisfies the marginal-regret balance 
$r_{\mathrm{RL}}(\tau^*) = r_{\mathrm{DM}}(\tau^*)$, and the resulting hybrid strategy is provably superior to either backbone alone.

% To compare the disparate paradigms of sequential RL and generative DM, we first establish a unified evaluation metric. While RL typically maximizes cumulative reward and DM optimizes distributional matching (e.g., log-likelihood), we align both under the lens of \textit{regret}.

\begin{definition}[Unified Regret Metric]
\label{def:regret}
Let $f: \mathcal{X} \to \mathbb{R}$ denote the objective function to be minimized, where $\mathcal{X} \subseteq \{0,1\}^d$ is the $d$-dimensional discrete decision space. We define the regret $\mathcal{R}(p, q)$ of a solver's output distribution $p$ relative to an optimal target distribution $q$ as the difference in their expected objectives:
\begin{equation}
    \mathcal{R}(p, q) = \mathbb{E}_{x \sim p}[f(x)] - \mathbb{E}_{x \sim q}[f(x)], x \in \mathcal{X}.
\end{equation}
\end{definition}
For the RL solver, $q=\delta_{x^*}$ is the Dirac measure at the global optimum $x^*$. The regret simplifies to the absolute optimality gap: $\mathcal{R}_{RL} = \mathbb{E}_{x \sim \pi}[f(x)] - f(x^*)$. For the conditional DM, given a prefix $c$, $q(\cdot\mid c)$ denotes the conditional distribution of the optimal completion. The regret instantiates as $\mathcal{R}_{DM}(p, q \mid c) = \mathbb{E}_{x \sim p(\cdot\mid c)}[f(x)] - \mathbb{E}_{x \sim q(\cdot\mid c)}[f(x)]$, quantifying the expected suboptimality of the model's output distribution $p(\cdot\mid c)$ relative to $q(\cdot\mid c)$. 

%The regret $\mathcal{R}_{DM}(p, q \mid c) $ measures the generation error relative to the best possible completion feasible under the current prefix $c$, i.e. $\mathcal{R}_{DM}(p, q \mid c) = \mathbb{E}_{x \sim p(\cdot\mid c)}[f(x)] - \mathbb{E}_{x \sim q(\cdot\mid c)}[f(x)]$. To analyze the regret scaling for generative models, we adopt a standard regularity assumption for continuous relaxations of CO problems.

%To analyze the regret scaling for generative models, we adopt a standard regularity assumption for continuous relaxations of CO problems. \yuheng{Rewrite this sentence, why do we set this assumption, and is it reasonable?}

To analyze $\mathcal{R}_{DM}$, we require a mild regularity condition on the relaxed objective, i.e. Lipschitz continuity, a condition that is essentially automatic for standard CO problems like TSP and MIS.

\begin{assumption}[Regularity of Objective Landscape]
\label{assump:regularity}
The continuous relaxation $\tilde{f}: [0,1]^d \rightarrow \mathbb{R}$ of the objective function $f$ is \textit{$\ell$-Lipschitz continuous}. %This implies that the solution cost is stable under small perturbations in the decision space.
\end{assumption}

% \begin{assumption}[Regularity of Objective Landscape]
% \label{assump:regularity}
% The continuous relaxation of the objective function $f(x)$ is \textit{$\ell$-Lipschitz continuous} over the solution manifold. %This implies that the solution cost is stable under small perturbations in the decision space.
% \end{assumption}

\begin{remark}
    Any $C^1$ function on the compact domain $[0,1]^d$ is necessarily Lipschitz by the Extreme Value Theorem. In Appendix~\ref{sec:app_theory_assump}, we verify that standard CO objectives such as TSP and MIS satisfy this condition under their continuous relaxations.
\end{remark}

With this regularity in hand, we can connect the Kullback-Leibler (KL) divergence minimized during DM training to the optimization regret.

\begin{lemma}[The Geometry-Statistics Bridge]
\label{lemma:kl_to_regret}
Under Assumption \ref{assump:regularity}, the regret $\mathcal{R}_{DM}(p, q)$ is upper-bounded by the KL divergence:
\begin{equation}
\mathcal{R}_{DM}(p, q) \le \ell \cdot \mathcal{C} \sqrt{KL(q \| p)},
\end{equation}
where $\mathcal{C}$ is a constant determined by the transport properties of the target distribution $q$ (e.g., its concentration or the domain's diameter).
\end{lemma}

% \begin{proof}[Sketch]
% 1. Geometry: By Kantorovich-Rubinstein duality \citep{villani2008optimal}, $|\mathbb{E}_p[f] - \mathbb{E}_q[f]| \le \ell \cdot W_1(p, q)$.
% 2. Statistics: The $W_1$ distance is bounded by $\sqrt{KL}$ via a transportation cost inequality (e.g., Talagrand's $T_2$ for smoothed targets or Pinsker's for bounded domains). The complete proof is in Appendix \ref{sec:app_theory_bridge}.
% \end{proof}

The proof is in Appendix \ref{sec:app_theory_bridge}. Now turn to the regret scaling of the conditional DM. Since the prefix $c$ fixes a subset of decision variables, the model only generates over the remaining unsolved subspace. We show that the regret scales with the square root of the dimension of the unsolved subspace.
%We now analyze the error scaling of the conditional DM. Unlike unconditional generation, our model operates under ``hard'' constraints (masking, i.e. partial solution generated by RL). We prove that the regret of such a model is governed by the \textit{effective dimension} of the unsolved subspace.

% \begin{theorem}[Dimension-Dependent Regret of Conditional DM]
% \label{thm:dm_scaling}
% Given a prefix $c$, let $p$ and $q$ be the generated distribution of a conditional DM and the ground-truth optimal conditional distribution, respectively. Suppose the conditioning is enforced via an orthogonal projection operator $M \in \mathbb{R}^{d \times d}$ such that $Mx = c$. Let $d_{eff} = \text{rank}(I-M)$ denote the dimension of the \textbf{unsolved subspace}. Assuming the DM achieves a per-dimension $L^2$ score-matching error $\epsilon_{score}$ on the unsolved manifold, the regret $\mathcal{R}_{DM}(p, q \mid c)$ under Assumption 4.2 satisfies:
% \begin{equation}
%     \mathcal{R}_{DM}(p, q \mid c) \le \mathcal{O}\left( \ell \cdot \sqrt{d_{eff}} \cdot \epsilon_{score} \right),
% \end{equation}
% where $\ell$ is the Lipschitz constant of the relaxed objective $f$. This bound implies that the model's optimality gap contracts as the unsolved dimensionality $d_{eff}$ decreases.
% \end{theorem}

\begin{lemma}[Dimension-Dependent Regret of Conditional DM]
\label{thm:dm_scaling}
Given a prefix $c$, let $p$ and $q$ be the generated distribution of a conditional DM and the ground-truth optimal conditional distribution, respectively. Let $d_{eff}$ denote the dimension of the \textbf{unsolved subspace}, i.e., the number of decision variables
not fixed by the prefix $c$. Suppose the DM achieves a per-dimension $L^2$ score-matching error $\epsilon_{score}$ on the unsolved subspace, then the regret $\mathcal{R}_{DM}(p, q \mid c)$ under Assumption 4.2 satisfies:
\begin{equation}
    \mathcal{R}_{DM}(p, q \mid c) \le \mathcal{O}\left( \ell \cdot \sqrt{d_{eff}} \cdot \epsilon_{score} \right),
\end{equation}
where $\ell$ is the Lipschitz constant of the relaxed objective $f$. 
%This bound implies that the model's optimality gap contracts as the unsolved dimensionality $d_{eff}$ decreases.
\end{lemma}

The proof is in Appendix \ref{sec:appendix_proof_dm_scaling}. We next analyze the regret behavior of the RL solver. 
%In contrast, RL policies suffer from error propagation over the decision horizon. \yuheng{This should mention that we turn to the RL solver and its boundary analysis.}
%whereas the cumulative regret relative to the total number of training steps may be sublinear.

\begin{lemma}[Super-linear horizon scaling of RL Regret]
\label{lemma:rl_scaling}
For an RL policy $\pi$, the error accumulation relative to the \textbf{decision horizon} $h$ within a single episode typically exhibits super-linear scaling:
\begin{equation}
    \mathcal{R}_{RL}(h) \sim \mathcal{O}(h^{\alpha}), \quad \text{where } \alpha > 1.
\end{equation}

%so the marginal regret $r_{RL}(h) = \frac{d}{dh}\mathcal{R}_{RL}(h) \sim \mathcal{O}(h^{\alpha-1})$ is strictly monotonically increasing.
%Specifically, when linear function approximation is employed, the sensitivity to the horizon $h$ often dominates the error bound (e.g., $\mathcal{O}(h^2)$ or higher), implying that the \textbf{marginal regret} $r_{RL}(h) = \frac{d}{dh}\mathcal{R}_{RL}(h) \sim \mathcal{O}(h^{\alpha-1})$ is strictly monotonically increasing.
\end{lemma}

%Now we model the hybrid inference process as an optimal handover problem: determining the critical step $\tau$ to switch from the fast RL solver to the rigorous global DM solver. 
This scaling captures horizon-wise error accumulation in sequential construction and is motivated by standard finite-horizon RL regret bounds; see Appendix~\ref{app:rl_scaling} for details.
Having established that the DM's regret scales as $\mathcal{O}(\sqrt{d_{eff}})$ and the RL's regret as $\mathcal{O}(h^{\alpha})$, we observe that under the hybrid inference framework, the two solvers partition the same problem:  the RL covers $h$ steps, leaving $d_{eff} = d - h$ variables for the conditional DM. This trade-off naturally gives rise to an optimal handover problem—determining the critical step $\tau^* \in [0, d]$ at which to trigger the conditional DM to minimize the total hybrid regret of the overall CO problem.
%$\mathcal{L}(\tau) = \int_0^{\tau} r_{\text{RL}}(h)\,dh + \int_{\tau}^{d} r_{\text{DM}}(h)\,dh$, where $r_{\text{DM}}(h)$ is the marginal regret of the DM at the decision step $h$ in the overall problem construction.

\begin{theorem}[Existence and Uniqueness of the Optimal Trigger Step $\tau^*$]
\label{thm:optimal_handoff}
 %Define the hybrid inference loss $\mathcal{L}(\tau): [0, d] \to \mathbb{R}_+$ for a trigger step $\tau$ as the cumulative expected regret: $\mathcal{L}(\tau) = \int_{0}^{\tau} r_{\text{RL}}(h) dh + \mathcal{R}_{\text{DM}}(p, q\mid c_\tau)$, where $c_\tau$ is the solution prefix of length $\tau$ generated by RL. 
 Given a RL solver and a conditional DM with fixed $p$ and $q$, let $r_{\text{RL}}(h)$ and $r_{\text{DM}}(h)$ denote the marginal regret of the RL solver and the DM at decision step $h$ in the overall problem construction, respectively. Define the hybrid regret $\mathcal{L}(\tau) = \int_0^{\tau} r_{\text{RL}}(h)\,dh + \int_{\tau}^{d} r_{\text{DM}}(h)\,dh$. Under Lemma \ref{thm:dm_scaling} and Lemma \ref{lemma:rl_scaling}, there exists a \textbf{unique} optimal trigger step $\tau^* \in (0, d)$ that minimizes $\mathcal{L}(\tau)$:
\begin{equation}\label{eqn:optimal_trigger}
    r_{\text{RL}}(\tau^*) = r_{\text{DM}}(\tau^*),
\end{equation}
if the boundary conditions $r_{\text{RL}}(0) < r_{\text{DM}}(0)$ and $r_{\text{RL}}(d) > r_{\text{DM}}(d)$ are satisfied.
\end{theorem}

% \begin{theorem}[Existence and Uniqueness of the Optimal step $\tau^*$]
% \label{thm:optimal_handoff}
% Let $\mathcal{X} \subset \mathbb{R}^H$ be the decision space of a combinatorial optimization problem with horizon $H$. We define the hybrid inference loss $\mathcal{L}(\tau): [0, H] \to \mathbb{R}_+$ for a trigger step $\tau$ as the cumulative expected regret: $\mathcal{L}(\tau) = \int_{0}^{\tau} r_{\text{RL}}(h) dh + \mathcal{R}_{\text{DM}}(p, q\mid c_\tau)$, where $c_\tau$ is the solution prefix of length $\tau$ generated by RL. Under the error accumulation properties established in Lemma \ref{lemma:rl_scaling} and the subspace scaling laws in Lemma \ref{thm:dm_scaling}, there exists a \textbf{unique} optimal triggering step $\tau^* \in (0, H)$ that minimizes the total expected regret, characterized by:
% \begin{equation}\label{eqn:optimal_trigger}
%     r_{\text{RL}}(\tau^*) = -\frac{d}{d\tau}\mathcal{R}_{\text{DM}}(p,q \mid c_{\tau^*}),
% \end{equation}
% if the boundary conditions $r_{\text{RL}}(0) < -\mathcal{R}'_{\text{DM}}(0)$ and $r_{\text{RL}}(H) > -\mathcal{R}'_{\text{DM}}(H)$ are satisfied.
% \end{theorem}

\begin{remark}%[Physical Interpretation of Boundaries]
The condition $r_{\text{RL}}(0) < r_{\text{DM}}(0)$ ensures that RL is initially more efficient than invoking a high-dimensional DM, while $r_{\text{RL}}(d) > r_{\text{DM}}(d)$ implies that the compounding error of RL eventually exceeds the marginal precision gain of the DM as the unsolved dimensionality $d_{eff}$ vanishes. 
\end{remark}

\begin{remark}[A General Sufficient Condition of Theorem \ref{thm:optimal_handoff}] 
The specific error-scaling assumptions considered in
Lemmas~\ref{thm:dm_scaling} and~\ref{lemma:rl_scaling} provide one sufficient condition for Theorem \ref{thm:optimal_handoff}, but not necessary conditions. More generally, the conclusion holds whenever the marginal-regret difference $\mathcal{L}'(\tau) := r_{\mathrm{RL}}(\tau)-r_{\mathrm{DM}}(\tau)$ is continuous and strictly increasing on $[0,d]$, with $\mathcal{L}'(0)<0<\mathcal{L}'(d)$. The proof of
Theorem~\ref{thm:optimal_handoff} and empirical analysis of regret curves are provided in Appendix \ref{sec:app_proof_optimal_trigger}.
    
\end{remark}

A direct consequence of Theorem \ref{thm:optimal_handoff} is that \textbf{the hybrid framework achieves strictly lower expected regret than either pure RL or pure DM alone}. However, $\tau^*$ is defined at the expected-regret level and is not directly observable on a single trajectory, which motivates designing a practical lightweight trigger to estimate the regime transition in Section \ref{sec:hyco}.

\section{HyCO: Single-Trigger Hybrid Inference Algorithm}
\label{sec:hyco}

The theoretical analysis in Section \ref{sec:theory_structure} yields two insights for algorithm design: (i) a single handover from RL to DM is sufficient to achieve strictly lower regret than either backbone alone, and (ii) the optimal trigger step occurs where the marginal regret of RL equals that of DM. HyCO operationalizes these insights through two components: a single-trigger inference pipeline that executes exactly one handover from RL to DM (Section~\ref{sec:hyco_pipeline}), and a lightweight adaptive trigger that detects trajectory-level symptoms of the regime transition as a practical proxy for the theoretical optimum (Section~\ref{sec:hyco_trigger}).

%Section~\ref{sec:theory_structure} characterizes hybrid inference as an optimal time of triggering problem and shows that, under the unified error-scaling model, there exists a unique optimal trigger point $\tau^*$ in expectation (Theorem \ref{thm:optimal_handoff}). Importantly, $\tau^*$ is an expectation-level theoretical object: it is defined by a marginal-regret equilibrium under the data distribution (Eq. \ref{eqn:optimal_trigger}). It does not correspond to a trajectory-level quantity that can be computed during inference. Hence, HyCO does not attempt to approximate $\tau^*$ on each instance. Instead, we implement the theory-implied \emph{single-trigger} pipeline and adopt a lightweight trigger as a statistically effective proxy for deciding when to invoke conditional global completion. We first present the inference pipeline in Section \ref{sec:hyco_pipeline} and then specify the lightweight trigger used as a practical decision rule for triggering in Section \ref{sec:hyco_trigger}.

\subsection{Single-Trigger inference pipeline}
\label{sec:hyco_pipeline}

Given a CO instance, HyCO starts with a fast sequential RL solver that constructs a partial solution (prefix) autoregressively.
At step $k$, the RL policy $\pi_\theta(\cdot | s_k)$ outputs a distribution over feasible actions $\mathcal{A}_k$,
and a greedy action is selected to extend the prefix.
HyCO continuously rolls out the RL solver until a trigger decision is made by a lightweight trigger (Section~\ref{sec:hyco_trigger}).
Once the trigger fires at some step $k^*$, HyCO invokes the conditional diffusion solver once and starts a dual-track generation.

\noindent\textbf{Dual-track solution generation and final selection.}
At the triggering step $k^*$, HyCO adopts a dual-track generation strategy to improve candidate diversity under a fixed compute budget: 
(i) hybrid path correction and (ii) full diffusion proposal generation.

For the hybrid path correction track, the diffusion model is used to \textit{correct the immediate next action} for the hybrid path $l_H$.
Instead of the RL model's greedy choice $a_{k^*}$, we select a corrected action $a_{k^*}^{DM}$ at step $k^*$ from the best solution in the DM candidate pool $\mathcal{P}_{DM}$.
This corrected action is appended to the tour, and the fast RL policy then autoregressively completes the remaining steps.

In parallel, HyCO uses the conditional DM to generate a set of complete, high-quality candidate solutions $\mathcal{P}_{DM}$.
Specifically, we take high-probability candidate actions from the RL policy at step $k^*$, form prefixes with each candidate, and run conditional diffusion completion to obtain full tours.

Finally, HyCO returns the best solution among the hybrid output and the diffusion proposals:
\begin{equation}
l_{\text{out}} = \arg\min\nolimits_{l \in \{l_H\}\cup \mathcal{P}_{DM}} \mathrm{Cost}(l).
\end{equation}

The motivation for dual-track generation is practical robustness:
Even if the trigger is slightly early or late, the hybrid track provides a low-cost corrected continuation, while the proposal pool increases the probability of recovering a high-quality completion. This parallel RL-based track also serves as a reliability anchor: if the DM fails to produce a feasible or high-quality completion, the hybrid path ensures that the system still returns a valid solution constructed by RL.

\subsection{Lightweight trigger as a trigger proxy}
\label{sec:hyco_trigger}

HyCO uses a lightweight trigger to decide \emph{when} to invoke conditional global completion. The principle follows from Theorem \ref{thm:optimal_handoff}: the regret-minimizing handover occurs when the marginal regret of extending the solution prefix with the RL solver matches the marginal regret of invoking the conditional DM for the remaining subproblem; beyond this point, further RL construction is expected to introduce more regret than DM completion. Since this comparison is defined at the expected-regret level and cannot be evaluated directly along a single trajectory, HyCO uses two observable trajectory-level proxies: policy entropy and RL-DM disagreement. Policy entropy serves as an internal reliability signal for RL: high entropy indicates that no feasible action is clearly dominant, making RL’s decision less trustworthy. Model disagreement serves as an external alignment signal: large disagreement indicates that the RL policy assigns high probability to candidate next actions that the DM probe considers less consistent with globally coherent solution structure. When either signal becomes large, further RL construction is likely to be less favorable, and HyCO therefore triggers conditional DM completion.

% Crucially, the trigger is not an optimal-stopping mechanism and is not designed to identify the expectation-level optimum $\tau^*$ on a per-instance basis. Instead, it acts as a low-cost proxy that detects \emph{trajectory-level symptoms} of a regime transition where global completion is often beneficial. We instantiate this proxy with two complementary signals. Policy entropy captures an internal symptom of unreliable RL continuation: as the sequential construction progresses, increasing uncertainty suggests that the policy can no longer confidently select a locally stable action. RL--DM disagreement captures an external symptom: a large mismatch between the RL policy and the DM-induced global prior suggests that the locally preferred continuation is becoming less consistent with globally coherent completions.

\noindent\textbf{Signal 1: RL policy uncertainty (internal).}
The first signal measures the RL policy's internal uncertainty via entropy:
\begin{equation}
H(\pi_\theta(\cdot \mid s_k))
=-\sum\nolimits_{a\in\mathcal{A}_k}\!\! \pi_\theta(a\mid s_k)\log \pi_\theta(a\mid s_k).
\label{eqn_entropy}
\end{equation}
High entropy indicates that the policy distribution is flat and no single action is clearly dominant,
suggesting a critical decision stage where greedy continuation is more likely to be unstable.

\noindent\textbf{Signal 2: Model disagreement (external).}
The second signal measures disagreement between the RL policy and a DM-induced global prior.
To obtain the diffusion prior efficiently (without running a full sampling trajectory at every step),
we probe the conditional diffusion model on a fixed \emph{Top-$M$} candidate set proposed by RL.
Let $\{a_k^{(i)}\}_{i=1}^M$ be the Top-$M$ actions under $\pi_\theta(\cdot| s_k)$, and let $c_k^{(i)}=c_k\oplus a_k^{(i)}$ be its one-step-extended prefixes.
We define an energy score $E_\phi(c_k^{(i)})$ from a single-step denoising probe that measures diffusion consistency with the proposed prefix structure:
\begin{equation}
    E_\phi(c_k^{(i)}) = - \frac{1}{| \mathcal{E}(c_k^{(i)}) |} \sum_{(u,v) \in \mathcal{E}(c_k^{(i)})} \log \left( \sigma \Big( D_{\phi}(X_{t_{\text{probe}}}, c_k^{(i)})_{uv} \Big) \right),
\label{eqn_energy}
\end{equation}
\noindent where $\mathcal{E}(c_k^{(i)})$ is the set of edges in the partial tour $c^{(i)}_k$, $\sigma(\cdot)$ is the sigmoid function, and $D_\phi(\cdot)_{uv}$ is the model's predicted logit for edge $(u,v)$.
Crucially, this summation acts as a mask: it specifically measures the diffusion model's confidence that the \textit{proposed structural connections} are valid parts of a global solution, ignoring non-connected pairs. We then convert energy values into a diffusion prior via a Boltzmann distribution:
\begin{equation}
q_{DM}(a_k^{(i)}\mid s_k)
=
\frac{\exp\big(-E_\phi(c_k^{(i)})/\tau_{\text{temp}}\big)}
{\sum_{j=1}^M \exp\big(-E_\phi(c_k^{(j)})/\tau_{\text{temp}}\big)}.
\label{eqn_prior_dis}
\end{equation}
We then compute the KL divergence between the RL distribution restricted to the same Top-$M$ support and the diffusion prior:
\begin{equation}
D_{KL}\!\left(\pi_\theta \,\|\, q_{DM}\right)
=
\sum_{i=1}^M
\pi_\theta(a_k^{(i)}\mid s_k)\,
\log \frac{\pi_\theta(a_k^{(i)}\mid s_k)}{q_{DM}(a_k^{(i)}\mid s_k)}.
\label{eqn_kl}
\end{equation}
A large divergence suggests that the RL policy’s locally preferred action conflicts with the DM’s global preference—a regime where invoking conditional global completion often yields gains.

\noindent\textbf{Triggering rule.}
In practice, HyCO adopts a simple and robust \textit{or} rule of triggering conditional DM once \emph{either} the RL policy becomes uncertain or the RL--DM disagreement becomes significant:
\begin{equation}
   H(\pi_\theta(\cdot\mid s_k)) > H_{\text{th}} \quad \lor \quad D_{KL}(\pi_\theta \| p_\phi) > D_{\text{th}}
\label{eq:or_trigger}
\end{equation}
The \textit{or} rule is easy to tune and empirically stable: high entropy indicates a locally ambiguous decision stage,
while a large KL divergence signals a confident but globally inconsistent choice where conditional completion is likely to help.
%HyCO performs \emph{at most one} trigger: once triggered, the conditional diffusion completion is invoked and no further triggering is allowed, ensuring a predictable compute budget and consistency with the single-trigger structure.

In all cases, HyCO executes \emph{at most one} trigger: after the conditional diffusion completion is invoked at $k^*$, no further triggering is allowed, ensuring a predictable compute budget and consistency with the single-trigger structure. Algorithm~\ref{alg:hyco} summarizes the full HyCO inference procedure.

\begin{algorithm}[tb]
\caption{HyCO: Hybrid Inference for CO }
\label{alg:hyco}
\begin{algorithmic}[1]
\REQUIRE RL policy $\pi_\theta$, conditional DM denoiser $D_\phi$,
thresholds $H_{\text{th}}, D_{\text{th}}$, Top-$P$ threshold $P_{\text{th}}$ %, probe size $M$
\STATE Initialize hybrid path $l_H \gets ()$, DM proposals $\mathcal{P}_{DM} \gets \emptyset$,
state $s_0$, \texttt{triggered} $\gets \texttt{false}$

\FOR{$k = 0, \ldots, N-1$}
    \IF{\textbf{not} \texttt{triggered}}
        %\STATE \textit{// Trigger evaluation}
        \STATE $H_k \gets H(\pi_\theta(\cdot \mid s_k))$ \textit{ via Eq.~(\ref{eqn_entropy})}
        \STATE Select Top-$M$ actions $\{a^{(i)}\}_{i=1}^M$ from $\pi_\theta(\cdot | s_k)$
        %\STATE Compute denoising energy $E_\phi(s_k \oplus a^{(i)})$ for each candidate \textit{ via Eq.~(\ref{eqn_energy})}
        \STATE Derive diffusion prior $q_{DM}(\cdot | s_k)$ \textit{via Eq.~(\ref{eqn_energy}, \ref{eqn_prior_dis})}
        \STATE $D_{KL,k} \gets D_{KL}(\pi_\theta \,\|\, q_{DM})$ \textit{via Eq.~(\ref{eqn_kl})}

        \IF{$H_k > H_{\text{th}} \ \textbf{or}\ D_{KL,k} > D_{\text{th}}$}
            %\STATE \textit{// Single handoff with dual-track generation}
            \STATE Select action set $\{a^*\}$ by cumulative probability threshold $P_{\text{th}}$
            \FORALL{selected actions $a^*$}
                \STATE Form candidate prefix $c^* \gets s_k \oplus a^*$
                \STATE Get full tour $l_{DM} \sim \text{DM\_Sampler}(D_\phi, c^*)$ %\chen{why use  $\pi_{DM}$? Confusing as  $\pi$ is also RL policy}
                \STATE $\mathcal{P}_{DM} \gets \mathcal{P}_{DM} \cup \{l_{DM}\}$
            \ENDFOR
            \STATE \texttt{triggered} $\gets \texttt{true}$ %\textit{// enforce single-handoff}
            \STATE Select $a_k^{DM}$ from the best $l \in \mathcal{P}_{DM}$ as the RL action: $a_k \gets a_k^{DM}$
        \ELSE
            \STATE $a_k \gets \arg\max_a \pi_\theta(a | s_k)$ %\textit{// no trigger: RL greedy}
        \ENDIF

    \ELSE
        \STATE $a_k \gets \arg\max_a \pi_\theta(a | s_k)$ %\textit{// after trigger: RL finishes $\pi_H$}
    \ENDIF

    \STATE Append $a_k$ to hybrid path $l_H$ and update to $s_{k+1}$
\ENDFOR
\STATE \textbf{return} $l_{\text{out}}$
\end{algorithmic}
\end{algorithm}
\vspace{-1.0em}

\section{Experiments}
We evaluate HyCO on standard benchmarks with a primary focus on the 2D Euclidean TSP, MIS problem, and Orienteering Problem (OP) at multiple scales. Our experiments are designed to validate two aspects: (i) Does HyCO consistently improve solution quality over both RL and DM backbones (Section \ref{sec:result vs backbones})? and (ii) Is the trigger step statistically aligned with near-optimal triggering regimes as measured by oracle grid search (Section \ref{sec:oracle_validation})? 

%\yuheng{How about using discussion to analyze some subtle findings figured out in the ablation study? Like, depending on the performance of backbones.}

\subsection{Experimental Settings}
\textbf{Datasets and Baselines.} We use standard public datasets for TSP with 50, 100, 500, and 1000 nodes. Problem instances are generated by sampling node coordinates uniformly from the unit square $[0, 1]^2$. We use the Concorde solver \citep{Applegate2006} to obtain optimal solutions for training our conditional DM and for calculating optimality gaps. %MIS and OP's datasets are stated in Appendix \ref{sec:app_exp_results_mis} and \ref{sec:app_exp_results_op}.
We compare our algorithm with a set of methods: exact solver Concorde, heuristic solver LKH-3 \citep{helsgaun2017extension}, and state-of-the-art learning-based approaches, including RL models (AM \citep{Kool2019}, POMO \citep{Kwon2020}, LEHD \citep{luo2023neural}), and DM solvers (DIFUSCO \citep{Sun2023}, T2T \citep{li2023t2t}, CoExpander \citep{macoexpander}).

% \textbf{Baselines.} We compare our algorithm with a comprehensive set of methods: exact solver Concorde, heuristic solver LKH-3 \citep{helsgaun2017extension}, and state-of-the-art learning-based approaches, including RL models (AM \citep{Kool2019}, POMO \citep{Kwon2020}, LEHD \citep{luo2023neural}), and DM solvers (DIFUSCO \citep{Sun2023}, T2T \citep{li2023t2t}).

%\textbf{Baselines.} We compare our algorithm with a comprehensive set of methods: exact solver Concorde, heuristic solver LKH-3 \citep{helsgaun2017extension}, and state-of-the-art learning-based approaches, including RL models (AM \citep{Kool2019}, POMO \citep{Kwon2020}, LEHD \citep{luo2023neural}), supervised models (GCN \citep{Joshi2019}), and DM solvers (DIFUSCO \citep{Sun2023}, T2T \citep{li2023t2t}).

% \textbf{Implementation Details.} Our framework orchestrates an RL policy and a conditional DM. For RL backbones, we use official pre-trained checkpoints. Our conditional DM, termed \textbf{Prefix-Difusco}, adapts the GNN architecture from DIFUSCO and is specifically trained with a masked loss and a curriculum learning strategy to complete tours from given prefixes. Full architectural details, training procedures, and decoding algorithms are provided in Appendix \ref{sec:app_implement_details}.

\textbf{Implementation Details.} HyCO uses official pre-trained RL checkpoints and a conditional DM, termed Prefix-Difusco, adapted from DIFUSCO's GNN architecture and trained with masked loss and curriculum learning for prefix-conditioned completion. Full details are provided in Appendix~\ref{sec:app_implement_details}.

\textbf{Evaluation Metrics.} We report three key metrics: the average tour length (Len), the percentage optimality gap to the exact solver's solutions (Gap), and the average solution time (Time).

\subsection{Performance Gains over Backbone Solvers}
\label{sec:result vs backbones}

\noindent\textbf{Small-scale TSP.} Table \ref{tab:tsp100_results} reports the performance on TSP-50 and TSP-100. HyCO achieves near-optimal solutions, consistently outperforming both its RL and conditional DM backbones. Compared to RL solvers, HyCO reduces the optimality gap by over \textbf{95\%} (from 1.76\% / 0.64\% to 0.03\% on TSP-50). Similarly, HyCO also improves upon the standalone Prefix\_Difusco, reducing the gap from 0.27\% to 0.03\%. Moreover, HyCO outperforms Prefix\_Difusco (S), which requires 16 parallel samples, with 4× fewer inference passes—confirming that a well-timed trigger is more effective than brute-force over-sampling. Notably, HyCO surpasses strong full-diffusion baselines like DIFUSCO and remains highly competitive with T2T.

% 导言区加: \usepackage{wrapfig}

\begin{wraptable}{r}{0.55\textwidth}
\vspace{-1.5em}
\centering
\footnotesize
\setlength{\tabcolsep}{4pt}
\renewcommand{\arraystretch}{1.1}
\caption{Results with Greedy Decoding on TSP-50 and TSP-100. $^*$ denotes results that are quoted from previous works. The ``Prefix\_Difusco (S)`` represents the best result obtained by running our Prefix\_Difusco model sampling 16 times in parallel.}
\label{tab:tsp100_results}
\begin{tabular}{lccccc}
\toprule
\textbf{ALGORITHM} & \textbf{TYPE} & \multicolumn{2}{c}{\textbf{TSP-50}} & \multicolumn{2}{c}{\textbf{TSP-100}} \\
\cmidrule(lr){3-4} \cmidrule(lr){5-6}
& & Len $\downarrow$ & Gap $\downarrow$ & Len $\downarrow$ & Gap $\downarrow$ \\
\midrule
Concorde & Exact & 5.69 & 0.00\% & 7.76 & 0.00\% \\
\midrule
AM* & RL & 5.80 & 1.76\% & 8.12 & 4.53\% \\
POMO* & RL & 5.73 & 0.64\% & 7.84 & 1.07\% \\
\midrule
DIFUSCO ($T_s$=50) & DM & 5.71 & 0.45\% & 7.85 & 1.21\% \\
T2T ($T_s$=50,$T_t$=15) & DM & 5.69 & 0.07\% & 7.77 & 0.20\% \\
Prefix\_Difusco & DM & 5.70 & 0.27\% & 7.81 & 0.62\% \\
Prefix\_Difusco (S) & DM & 5.69 & 0.04\% & 7.76 & 0.04\% \\
\midrule
\textbf{HyCO (AM)} & \textbf{Hy} & 5.69 & \textbf{0.03\%} & 7.76 & \textbf{0.02\%} \\
\textbf{HyCO (POMO)} & \textbf{Hy} & 5.69 & \textbf{0.03\%} & 7.76 & \textbf{0.02\%} \\
\bottomrule
\end{tabular}
\vspace{-1em}
\end{wraptable}

% \begin{table}[h]
% \centering
% %\small
% %\scriptsize
% \footnotesize
% \setlength{\tabcolsep}{4pt}
% \renewcommand{\arraystretch}{1.1} 
% \caption{Results with Greedy Decoding on TSP-50 and TSP-100. $^*$ denotes results that are quoted from previous works. ``Prefix\_Difusco`` is the conditional DM used in HyCO. The ``Prefix\_Difusco (S)`` represents the best result obtained by running our Prefix\_Difusco model sampling 16 times in parallel. \yuheng{This table should be reconsidered what data needs to be put in.}}
% \label{tab:tsp100_results}

% \begin{tabular}{lccccc}
% \toprule
% \textbf{ALGORITHM} & \textbf{TYPE} & \multicolumn{2}{c}{\textbf{TSP-50}} & \multicolumn{2}{c}{\textbf{TSP-100}} \\
% \cmidrule(lr){3-4} \cmidrule(lr){5-6}
% & & Len $\downarrow$ & GAP $\downarrow$ & Len $\downarrow$ & GAP $\downarrow$ \\
% \midrule
% Concorde & Exact & 5.69 & 0.00\% & 7.76 & 0.00\% \\
% \midrule
% AM* & RL & 5.80 & 1.76\% & 8.12 & 4.53\% \\
% POMO* & RL & 5.73 & 0.64\% & 7.84 & 1.07\% \\
% \midrule
% DIFUSCO ($T_s$=50) & DM & 5.71 & 0.45\% & 7.85 & 1.21\% \\
% T2T ($T_s$=50,$T_t$=15) & DM & 5.69 & 0.07\% & 7.77 & 0.20\% \\
% Prefix\_Difusco & DM & 5.70 & 0.27\% & 7.81 & 0.62\% \\
% Prefix\_Difusco (S) & DM & 5.69 & 0.04\% & 7.76 & 0.04\% \\
% \midrule
% \textbf{HyCO (AM)} & \textbf{Hy} & 5.69 & \textbf{0.03\%} & 7.76 & \textbf{0.02\%} \\
% \textbf{HyCO (POMO)} & \textbf{Hy} & 5.69 & \textbf{0.03\%} & 7.76 & \textbf{0.02\%} \\
% \bottomrule
% \end{tabular}
% \end{table}

\noindent\textbf{Large-scale TSP.} Table \ref{tab:tsp_large_results} evaluates HyCO on TSP-500 and TSP-1000, where long-horizon sequential construction becomes substantially more challenging. With POMO as the RL backbone, HyCO substantially improves over the standalone POMO solver, reducing its optimality gap by 12.06 percentage points on TSP-500 and 10.80 percentage points on TSP-1000, corresponding to relative reductions of 77.2\% and 70.2\%, respectively. HyCO also reduces the gap of the Prefix\_Difusco and Prefix\_CoExpander baselines by more than half on both scales, showing that the gain comes from combining an informative RL prefix with conditional diffusion completion rather than from either component alone. The benefit further persists with LEHD, a stronger backbone designed for large-scale neural CO: although the gains are more modest, HyCO still achieves the best overall gaps in the table. These results suggest that HyCO can enhance different RL backbone regimes, including cases where both standalone RL and standalone conditional DM remain far from optimal. Additional 2-opt results in Appendix \ref{sec:app_exp_results_tsp} show that HyCO maintains its relative advantage after local refinement.

\begin{table}[h!]
\vspace{-1.0em}
\centering
\small
%\scriptsize
\setlength{\tabcolsep}{8pt} % 适当减小列间距
\caption{Results on large-scale TSP problems. RL, SL, G, and S denote Reinforcement Learning, Supervised Learning, Greedy decoding, and Sample decoding, respectively. * indicates the baseline for computing the performance gap. The TIME reports the average inference time per instance.} 
\label{tab:tsp_large_results}
\renewcommand{\arraystretch}{1.10} 

\begin{tabular}{llcccccc}
\toprule
\centering

%\scriptsize
\setlength{\tabcolsep}{12pt} % 适当减小列间距
\textbf{ALGORITHM} & \textbf{TYPE} & \multicolumn{3}{c}{\textbf{TSP-500}} & \multicolumn{3}{c}{\textbf{TSP-1000}} \\
\cmidrule(lr){3-5} \cmidrule(lr){6-8}
& & Len$\downarrow$ & Gap$\downarrow$ & Time$\downarrow$ & Len$\downarrow$ & Gap$\downarrow$ & Time$\downarrow$ \\
\midrule
Concorde        & Exact        & 16.55* & 0.00\%   & 37.66m & 23.53* & 0.00\%   & 6.65h  \\
Gurobi          & Exact        & 16.55  & 0.00\%   & 45.63h & 23.53  & 0.00\%   & 48h   \\
LKH-3 (default) & Heuristics   & 16.55  & 0.00\%   & 46.28m & 23.53  & 0.00\%   & 2.57h  \\
\midrule
AM              & RL+G         & 20.02  & 20.99\%  & 0.47s  & 28.52  & 21.21\%  & 1.09s \\
GCN             & SL+G         & 29.72  & 79.61\%  & 6.67m  & 43.15  & 83.38\%  & 28.52m  \\
POMO     & RL+G      & 19.13  & 15.62\%  & 0.49s & 27.15   & 15.38\%        & 0.94s     \\
% POMO(S)     & RL+S      &  18.59  & 12.32\%  & 2.97m  &   28.34    &  20.44\%      &  3.25m   
% \\
DIMES           & RL+G         & 18.93  & 14.38\%  & 0.97m  & 27.23  & 15.73\%  & 2.08m  \\
% DIMES           & RL+AS+G      & 17.81  & 7.61\%   & 2.10h  & 25.11  & 6.72\%   & 4.49h  \\
LEHD           & RL/SL+G    &   16.82  & 1.64\%  &  1.82s  &  24.29 & 3.23\%   &   3.84s\\
DIFUSCO         & SL+G         & 18.11  & 9.41\%   & 5.70m  & 25.68  & 9.14\%   & 11.5m  \\
T2T             & SL+G         & 17.39  & 5.09\%   & 4.90m  &25.17  & 8.87\%   & 15.66m  \\
Prefix\_Difusco  & SL+G         & 17.92  & 8.23\%   & 0.14m  & 25.86  & 9.91\%   & 0.35m  \\
Prefix\_Difusco(S)  & SL+S         & 17.31  & 4.58\%   & 2.10m  & 24.96  & 6.07\%   & 5.94m  \\
Prefix\_CoExpander    & SL+G         & 17.61  & 6.43\%   & \textbf{0.13s}  &25.47  & 8.26\%   & \textbf{0.81s}  \\
\textbf{HyCO(POMO+Difu)}           & \textbf{RL+SL+G} & 17.24 & 3.56\% & 1.20m & 24.61 & 4.58\% & 2.43m \\
\textbf{HyCO(LEHD+Difu)}           & \textbf{RL+SL+G} & \textbf{16.81} & \textbf{1.57\%} & 0.35m & 24.24 & 3.04\% & 0.36m \\
\textbf{HyCO(LEHD+CoEx)}           & \textbf{RL+SL+G} & \textbf{16.81} & \textbf{1.57\%} & 5.32s & \textbf{23.79} & \textbf{1.09\%} & 5.89s \\
% \midrule
% POMO     & RL+G+2OPT      & 17.77  & 7.37\%  & 0.58s & 25.10   & 6.67\%        & 1.05s     \\
% DIMES              & RL+G+2OPT      & 17.65  & 6.62\%  & 1.01m & 24.83  & 7.38\%  & 2.29m \\ 
% DIMES              & RL+AS+G+2OPT   & 17.31  & 4.57\%  & 2.10h & 24.33  & 5.22\%  & 4.49h \\
% LEHD           & RL/SL+G+2OPT    &   16.77  & 1.40\%  &  1.83s  & 24.06& 2.29\%   &   3.85s\\
% DIFUSCO                & SL+G+2OPT        & 16.83  & 1.68\%  & 5.75m  &23.92  & 1.66\%  & 17.52m \\
% T2T                    & SL+G+2OPT        &16.86  & 1.92\%  & 2.42m  &23.86  &1.42\%  & 15.90m \\
%Prefix\_Difusco(used in HyCO)  & SL+G+2OPT         & 16.83  & 1.68\%   & 0.56m  & 23.93  & 1.72\%   & 0.52m  \\
% \textbf{HyCO(POMO)}           & \textbf{RL+SL+G+2OPT} & 16.82 & 1.63\% & 2.32m & 23.92 & 1.66\% & 4.40m \\
% \textbf{HyCO(LEHD)}           & \textbf{RL+SL+G+2OPT} & \textbf{16.74} & \textbf{1.2\%} & 0.36m & \textbf{23.64} & \textbf{0.46\%} & 0.36m \\
\bottomrule
\end{tabular}
%\vspace{-1em}
\end{table}

\begin{table*}[htbp!]
\vspace{-0.5em}
\centering
\caption{Performance comparison across RB-LARGE, ER-700-800, and SATLIB benchmarks for the MIS problem. Obj. denotes the average node number of independent sets (higher is better). Gap is the percentage deviation from the optimal solution (lower is better). Time is the average inference time in seconds. S: Sample Decoding. $S_{cand}$ is the candidate number of full DM proposals $\mathcal{P}_{DM}$.}
\label{tab:mis_results}
\small
%\scriptsize               % 保持小字体
%\footnotesize
\setlength{\tabcolsep}{4pt}  
\begin{tabular}{l r r r r r r r r r}
\toprule
\textbf{METHOD} & \multicolumn{3}{c}{\textbf{RB-LARGE}} & \multicolumn{3}{c}{\textbf{ER-700-800}} & \multicolumn{3}{c}{\textbf{SATLIB}} \\
\cmidrule(lr){2-4} \cmidrule(lr){5-7} \cmidrule(lr){8-10} 
 & Obj.$\uparrow$ & Gap$\downarrow$ & Time$\downarrow$ & Obj.$\uparrow$ & Gap$\downarrow$ & Time$\downarrow$ & Obj.$\uparrow$ & Gap$\downarrow$ & Time$\downarrow$ \\
\midrule
KaMIS & 43.00 & 0.00\% & 56.97s & 44.97 & 0.00\% & 60.75s & 425.95 & 0.00\% & 24.37s \\
Gurobi & 42.19 & 1.83\% & 33.84s & 38.78 & 13.75\% & 60.49s & 425.92 & 0.01\% & 13.47s \\
\midrule
LWD (RL) & 36.67  & 15.31\% &\textbf{0.85s} & 39.05 & 13.16\% & \textbf{ 0.64s }& 421.80 & 1.04\%  & 0.66s \\
prefix\_Coexpander (Greedy) &  40.05& 6.84\% & 1.15s & 40.18 & 10.65\% & 0.71s & 421.84 & 0.96\% & \textbf{0.25s} \\
prefix\_Coexpander (s=5) & 40.06 & 6.83\% & 1.22s & 40.85 & 9.16\% & 0.96s & 424.17 & 0.42\% & 0.91s \\
prefix\_Coexpander(s=64)  & 40.58 & 5.63\% & 4.48s & 40.91 & 9.03\% & 4.03s &   424.78& 0.27\% & 4.27s \\
\textbf{HyCO} ($S_{cand}$=5) & \textbf{40.73} & \textbf{5.28\%} & 2.69s & \textbf{42.38} & \textbf{5.76\%} &  4.69s &  \textbf{425.40} & \textbf{ 0.12\%} & 6.31s \\
\bottomrule
\end{tabular}
\setlength{\tabcolsep}{6pt} % 恢复默认列间距
%\vspace{-1.5em}
\end{table*}

Beyond TSP, Table \ref{tab:mis_results} shows that HyCO also transfers to MIS. Averaged over RB-LARGE, ER-700-800, and SATLIB, HyCO reduces the optimality gap by 6.12\% over the RL baseline LWD. Under the same small candidate budget, it also improves over prefix\_Coexpander by an average of 1.75\%, and remains better than the heavier prefix\_Coexpander (s=64) baseline. A more detailed per-benchmark analysis of MIS is provided in Appendix \ref{sec:app_exp_results_mis}. OP results in Appendix \ref{sec:app_exp_results_op} show a similar trend on another routing problem: averaged over OP-50/100/200, HyCO reduces the optimality gap by 8.21\% over AM and by 3.53\% over the DM backbone. Together, these results support that the single-trigger hybrid principle extends beyond TSP to both graph-structured MIS and another routing problem OP.

% We further evaluate HyCO on the MIS and OP and report additional results in Appendix \ref{sec:app_exp_results_mis} and \ref{sec:app_exp_results_op}. HyCO remains consistently effective beyond TSP benchmarks, suggesting that the single-trigger hybrid principle is not specific to TSP.

\subsection{Oracle Trigger Validation: Adaptive vs.\ Instance-wise Best Fixed}
\label{sec:oracle_validation}

To evaluate whether the adaptive trigger can approximate a near-optimal switching time in practice, we compare it with an instance-wise oracle trigger obtained by grid search. We run the same HyCO pipeline over a grid of fixed trigger steps and define the best-performing one as the oracle trigger. Let $l_{\text{oracle}}$ and $l_{\text{adapt}}$ denote the solutions generated by the oracle and adaptive triggers, respectively.
We compare them using the solution gap
$\Delta_{\text{cost}} = |\mathrm{Cost}(l_{\text{adapt}})-\mathrm{Cost}(l_{\text{oracle}})|$
and the trigger-step distance
$\Delta_{\text{step}} = |\tau_{\text{adapt}}-\tau_{\text{oracle}}|$.
To account for different problem sizes, we also report the normalized distance
$\Delta_{\text{step}}^{\mathrm{r}}=\Delta_{\text{step}}/N$, where $N$ is the problem size.

% Section~\ref{sec:hyco} emphasizes that the optimal trigger step $\tau^\ast$ in Theorem~\ref{thm:optimal_handoff}
% is an expectation-level object and is not computable on individual trajectories.
% We therefore evaluate the proposed trigger from a \emph{practical} perspective:
% whether it behaves as a statistically effective proxy for selecting a near-optimal triggering regime \emph{per instance}.

% \noindent\textbf{Oracle protocol.}
% For each test instance, we perform a grid search over a set of fixed trigger steps $\mathcal{T}=\{\tau_1,\dots,\tau_K\}$ and run the same HyCO pipeline under each fixed $\tau$. We treat the best-performing fixed trigger $\tau_{\text{oracle}}$ as an \emph{oracle} reference and its final solution generated HyCO as $l_{\text{oracle}}$.

% We then compare HyCO's adaptive trigger $\tau_{\text{adapt}}$ against $\tau_{\text{oracle}}$ from two aspects:
% (i) performance gap $\Delta_{\text{cost}} = \big|\mathrm{Cost}(l_{\text{adapt}})-\mathrm{Cost}(l_{\text{oracle}})\big|$, and (ii) trigger step distance $\Delta_{\text{step}} \;=\; |\tau_{\text{adapt}}-\tau_{\text{oracle}}|$.
% To make step distances comparable across problem scales, we measure the \emph{relative step distance} normalized by the problem scale $N$: $\Delta_{\text{step}}^{\text{r}} \;=\; \frac{|\tau_{\text{adapt}}-\tau_{\text{oracle}}|}{N}$.
% This normalization is important because SATLIB and ER graphs contain substantially more candidate nodes than TSP, so an absolute deviation in triggering steps may translate to a much smaller fraction of the overall construction horizon.

\begin{wraptable}{r}{0.5\textwidth}
\vspace{-1.5em}
\centering
\small
\setlength{\tabcolsep}{4pt}
\renewcommand{\arraystretch}{1.15}
\caption{Oracle trigger validation on TSP/MIS. %``Oracle'' is the instance-wise best fixed trigger found by grid search.
% We report absolute cost gap $\Delta_{\text{cost}}$, relative cost gap $\Delta_{\text{cost}}/f^\ast$,
% absolute step distance $\Delta_{\text{step}}$, and relative step distance $\Delta_{\text{step}}^{\text{rel}}=\Delta_{\text{step}}/N$.
}
\label{tab:oracle_validation}
\begin{tabular}{lcccc}
\toprule
\textbf{Dataset}&
$\boldsymbol{\Delta_{\text{cost}}}$ &
$\boldsymbol{\Delta_{\text{cost}}/f^\ast}$ &
$\boldsymbol{\Delta_{\text{step}}}$ &
$\boldsymbol{\Delta_{\text{step}}^{\text{r}}}$ \\
\midrule
TSP-100      & 0.0028 & 0.036\% & 1.48  & 1.48\% \\
TSP-500      & 0.3424 & 2.057\%  & 3.23  & 0.65\% \\
SATLIB       & 0.3812 & 0.089\% & 22.60 & 1.77\% \\
ER 700--800  & 1.0460 & 2.330\%  & 18.62 & 2.48\% \\
\bottomrule
\end{tabular}
%\vspace{-1em}
\end{wraptable}

%\noindent\textbf{Results Analysis.}
% Table~\ref{tab:oracle_validation} summarizes oracle validation results.
% Across all benchmarks, HyCO's trigger fires on nearly all instances,
% indicating that the proxy signals are sufficiently informative to activate global completion when needed.
% On TSP100, the adaptive trigger is extremely close to the instance-wise oracle, both in terms of solution quality and triggering step.
% On larger and structurally diverse graphs (SATLIB / ER700--800), the absolute step distance can be larger, yet the \emph{relative} distance remains small after normalization due to the longer construction horizon, and the performance gap remains minor.
% Overall, these results support the claim that our trigger is a \emph{statistically effective proxy}
% for identifying near-optimal triggering regimes.

\noindent\textbf{Results Analysis.}
Table~\ref{tab:oracle_validation} shows that the adaptive trigger closely tracks the instance-wise oracle in trigger location, with normalized step distances below \(2.5\%\) across all benchmarks. The resulting solution quality is also close to the oracle: the relative cost difference is only \(0.036\%\) on TSP-100, \(2.057\%\) on TSP-500, \(0.089\%\) on SATLIB, and \(2.330\%\) on ER-700--800. The slightly larger deviations on TSP-500 and ER-700--800 suggest that larger or structurally more complex instances can be more sensitive to small trigger-step shifts, even when the normalized trigger distance remains small. These results indicate that the adaptive trigger identifies a near-oracle switching region in practice, supporting entropy and RL--DM disagreement as effective trajectory-level proxies for adaptive triggering.
% Table~\ref{tab:oracle_validation} shows that the adaptive trigger closely tracks the instance-wise oracle in trigger location: the normalized step distance is below \(2.5\%\) across all benchmarks, and is especially small on TSP-100 and TSP-500 (\(0.83\%\) and \(0.65\%\)). The resulting solution quality is nearly oracle-matched on TSP-100, SATLIB, and ER-700--800, with relative cost differences of \(0.036\%\), \(0.286\%\), and \(2.33\%\), respectively. TSP-500 has a larger cost deviation (\(5.35\%\)), suggesting higher sensitivity to the exact trigger step on some large-scale instances; however, the small normalized step distance still indicates that the adaptive trigger locates a near-oracle switching region. Overall, these results support entropy and RL--DM disagreement as practical trajectory-level proxies for adaptive triggering. 

%Complementary experiments in Appendix~\ref{sec:app_ablation} further support this conclusion: adaptive triggering consistently outperforms validation-selected fixed-time schedules, while ablations verify the benefit of the dual-signal trigger and analyze the sensitivity to energy-probe and decoding hyperparameters.

Further experiments in Appendix~\ref{sec:app_exp_fixed_vs_adaptive} compare adaptive triggering with validation-selected fixed-time schedules, while Appendix~\ref{sec:app_trigger_stats} reports trigger statistics across datasets and problem scales. Appendix~\ref{sec:app_ablation} presents ablation studies on trigger-threshold calibration, energy-probe and decoding hyperparameters, the dual-signal trigger, the complementarity of the two output tracks, and the trade-off between solution quality and inference time under single and multiple triggering.

% Besides, the ablation experimental results are in Appendix \ref{sec:app_ablation} to show HyCO's robustness to hyperparameters and the necessity of dual signal design.

% \begin{table}[htbp]
% \centering
% \small
% \setlength{\tabcolsep}{4pt}
% \renewcommand{\arraystretch}{1.15}
% \caption{Oracle trigger validation. ``Oracle'' is the instance-wise best fixed trigger found by grid search.
% We report absolute cost gap $\Delta_{\text{cost}}$, relative cost gap $\Delta_{\text{cost}}/f^\ast$ (using the dataset-specific optimal value $f^\ast$),
% absolute step distance $\Delta_{\text{step}}$, and relative step distance $\Delta_{\text{step}}^{\text{rel}}=\Delta_{\text{step}}/N$.}
% \label{tab:oracle_validation}
% \begin{tabular}{lcccccc}
% \toprule
% \textbf{Dataset}&
% $\boldsymbol{\Delta_{\text{cost}}}$ &
% $\boldsymbol{\Delta_{\text{cost}}/f^\ast}$ &
% $\boldsymbol{\Delta_{\text{step}}}$ &
% $\boldsymbol{\Delta_{\text{step}}^{\text{rel}}}$ \\
% \midrule
% TSP-100    & 0.0028 & 0.036\% & 1.48  & 1.48\% \\
% TSP-500    & 0.8855 & 5.35\%  & 3.23  & 0.65\% \\
% SATLIB    & 1.2180 & 0.286\% & 22.60 & 1.77\% \\
% ER 700--800  & 1.0460 & 2.33\%  & 18.62 & 2.48\% \\
% \bottomrule
% \end{tabular}
% \end{table}

\section{Conclusion}

We introduced HyCO, a hybrid neural solver that combines sequential RL construction with conditional DM completion for CO. Motivated by the complementary regret behavior of RL and DM backbones, our unified regret analysis shows that a single RL-to-DM handover can achieve strictly lower expected regret than either backbone alone, and that the regret-minimizing handover occurs at the marginal-regret balance between the two solvers. Since this marginal-regret balance is not directly observable on individual trajectories, HyCO uses policy entropy and model disagreement as lightweight trajectory-level proxies for detecting the corresponding regime transition in practice. Experiments on TSP, MIS, and OP demonstrate consistent improvements over both backbones, while oracle validation and ablations support the effectiveness of the adaptive dual-signal trigger. While our theoretical analysis focuses on RL and DM backbones, the core insight—complementary error scaling between sequential and global solvers—may generalize to other solver pairings with analogous properties, which we leave for future work. More broadly, we hope this work opens a new direction for neural combinatorial optimization through adaptive collaboration among solvers with complementary strengths.

\section*{Acknowledgements and Disclosures}

\textbf{Funding.}
The work done by Y.\,Li and H. \,Chen was supported by the National Science Foundation under IIS-2348405, the Coastal Virginia Center for Cyber Innovation (COVA CCI), Commonwealth Cyber Initiative (CCI), the William \& Mary Faculty Research Grant, and generous gifts from NVIDIA and Amazon. The work done by D.\,Yang and Y.\,Xiong was supported by the Coastal Virginia Center for Cyber Innovation (COVA CCI) and the Commonwealth Cyber Initiative (CCI).

\textbf{Competing interests.}
The authors declare no financial relationships or competing interests outside the submitted work that could reasonably be perceived as influencing this work.

\newpage

\bibliography{references}
\bibliographystyle{unsrt} %plainnat ieeetr abbrvnat
%%%%%%%%%%%%%%%%%%%%%%%%%%%%%%%%%%%%%%%%%%%%%%%%%%%%%%%%%%%%

\clearpage
\section*{Technical Appendices and Supplementary Material}
\label{app:supplementary}

\startcontents[appendices]
\printcontents[appendices]{}{1}{
    \setcounter{tocdepth}{2}
    \vspace{-0.5em}
}
\appendix

\onecolumn
\section{Theoretical Proof}
\label{sec:app_theory}

\subsection{On the Validity and Applicability of Assumption \ref{assump:regularity}}
\label{sec:app_theory_assump}
To bridge the gap between discrete combinatorial structures and the continuous dynamics of Diffusion Models (DMs), we justify the $L$-Lipschitz continuity of the objective function $f$ under its continuous relaxation. This regularity is a fundamental requirement for the subsequent regret analysis.

In DMs, the discrete objective function $f: \{0,1\}^d \to \mathbb{R}$ is extended to the continuous unit hypercube $\mathcal{X} = [0,1]^d$. For any combinatorial problem, there exists a canonical \textbf{multilinear extension} $\tilde{f}: [0,1]^d \to \mathbb{R}$ defined as:
\begin{equation}
\tilde{f}(x) = \sum_{S \subseteq \{1,\dots,d\}} f(S) \prod_{i \in S} x_i \prod_{j \notin S} (1-x_j), \quad x \in \mathcal{X},
\end{equation}
where $S$ denotes a discrete candidate solution (e.g., a specific subset of edges in TSP), and $x_i$ represents the marginal probability of element $i$ being included in the solution, often represented as a probabilistic ``heatmap''. Mathematically, $\tilde{f}(x)$ represents the expected objective value $\mathbb{E}_{S \sim x}[f(S)]$ under independent Bernoulli sampling. Since $\tilde{f}$ is a polynomial in $x$, it is continuously differentiable ($C^1$) on the compact set $\mathcal{X}$.

%\noindent\textbf{2. Boundedness of Gradients and $L$-Lipschitz Continuity.}
By the Extreme Value Theorem, any continuously differentiable function on a compact set has a bounded gradient. Let $L = \sup_{x \in \mathcal{X}} \|\nabla \tilde{f}(x)\|$. Then, for any $x, y \in \mathcal{X}$, the Mean Value Theorem implies:
\begin{equation}
|\tilde{f}(x) - \tilde{f}(y)| \leq L \|x - y\|
\end{equation}
This confirms that the relaxed objective is $L$-Lipschitz continuous. We specifically highlight its applicability to the benchmarks considered:
\begin{itemize}
    \item Linear Objectives (e.g., TSP): The objective is $f(x) = \langle W, x \rangle$. The gradient $\nabla f = W$ is a constant edge-weight matrix, thus $L = \|W\|_F$ is finite and independent of $x$.
    \item Quadratic Objectives (e.g., MIS): Formulated as $f(x) = x^T Q x$, the gradient $\nabla f = 2Qx$ is linear in $x$. On the bounded domain $[0,1]^d$, the spectral properties of the graph matrix $Q$ ensure the gradient remains bounded, satisfying the $L$-Lipschitz condition.
\end{itemize}

The $L$-Lipschitz property is the minimal requirement to satisfy the Kantorovich-Rubinstein duality. It allows us to bound the optimization regret by the Wasserstein-1 distance ($W_1$), which is subsequently connected to the KL divergence matching the DM training objective via Talagrand's inequality in Lemma \ref{lemma:kl_to_regret}.

\subsection{Proof of Lemma \ref{lemma:kl_to_regret}}
\label{sec:app_theory_bridge}

% \begin{lemma}[The Geometry-Statistics Bridge]
% Under Assumption , the regret $\mathcal{R}(p, q)$ is upper-bounded by the coupling of the landscape complexity $L$ and the KL divergence:
% \begin{equation}
% \mathcal{R}(p, q) \le L \cdot \mathcal{C} \sqrt{KL(q || p)},
% \end{equation}
% where $\mathcal{C}$ is a constant determined by the transport properties of the target distribution $q$ (e.g., its concentration or the domain's diameter).
% \end{lemma}

\begin{proof}
    We establish this bound by decomposing the optimization regret into a geometric component (landscape regularity) and a statistical component (transportation cost).

    The regret $\mathcal{R}(p, q)$ is defined as the difference in expected values of the objective $f$ under distributions $p$ and $q$. Given that the continuous relaxation $\tilde{f}$ is $\ell$-Lipschitz (Assumption \ref{assump:regularity}), we invoke the Kantorovich-Rubinstein duality for the $W_1$ distance:
\begin{equation}
    \mathcal{R}(p, q) = |\mathbb{E}_{x \sim p}[\tilde{f}(x)] - \mathbb{E}_{x \sim q}[\tilde{f}(x)]| \le \sup_{|\phi|_{Lip} \le \ell} \left| \int \phi dp - \int \phi dq \right| = \ell \cdot W_1(p, q),
\end{equation}

where $\text{Lip}_{\ell}(\mathcal{X}) = \{ \phi: \mathcal{X} \to \mathbb{R} \mid |\phi|_{Lip} \le \ell \}$ denotes the set of all functions with Lipschitz constant at most $\ell$. By the definition of the $W_1$ distance, the supremum on the right-hand side is exactly $\ell \cdot W_1(p, q)$. This step translates the optimization gap into a geometric distance between distributions.

Then, to relate the geometric distance $W_1$ to the information-theoretic divergence $KL(q \| p)$, we consider two distinct regimes to ensure the universality of the lemma:

\begin{itemize}
\item \textbf{General Case (Pinsker's Inequality):} On a compact solution manifold $\mathcal{X}$ with diameter $D = \sup_{x,y \in \mathcal{X}} \|x-y\|$, the $W_1$ distance is bounded by the Total Variation (TV) distance: $W_1(p, q) = W_1(q,p) \le D \cdot TV(q, p)$. By Pinsker's inequality, $TV(q, p) \le \sqrt{\frac{1}{2} KL(q || p)}$, yielding:\begin{equation}W_1(p, q) \le \frac{D}{\sqrt{2}} \sqrt{KL(q \parallel p)}.\end{equation}This bound holds for any arbitrary target distribution $q$ on a bounded domain, justifying the lemma's applicability to general NCO tasks beyond diffusion-based priors.

\item \textbf{Regular Case (Talagrand's $T_2$ Inequality):} For target distributions $q$ that exhibit stronger concentration (e.g., Gaussian-smoothed Gibbs measures used in DM), we assume $q$ satisfies a Log-Sobolev Inequality with constant $1/\sigma^2$. This implies the Talagrand's $T_2$ inequality: $W_1(p, q) \le W_2(p, q) = W_2(q, p) \le \sqrt{2\sigma^2 KL(q \| p)}$.
\end{itemize}

By defining $\mathcal{C} = \min \{ \frac{D}{\sqrt{2}}, \sqrt{2\sigma^2} \}$ (or more generally, treating $\mathcal{C}$ as a constant capturing the domain diameter or concentration radius), we obtain:
\begin{equation}
\mathcal{R}(p, q) \le \ell \cdot \mathcal{C} \sqrt{KL(q || p)}.
\end{equation}
\end{proof}

\subsection{Proof of Lemma \ref{thm:dm_scaling}: Dimension-Dependent Regret of Conditional DM}
\label{sec:appendix_proof_dm_scaling}

\begin{proof}
    This proof establishes that the optimization regret of a conditional DM is governed by the intrinsic dimensionality of the unsolved decision space. We provide a rigorous derivation by projecting the dynamics onto the free subspace and utilizing the conditional covariance dissipation identity.

% \noindent\textbf{1. Subspace Confinement and Measure Decomposition.}
In our task, the conditioning $Mx = c$ restricts the generative process to an affine manifold $\mathcal{V}_c := \{x \in \mathbb{R}^d \mid Mx = c\}$, where $c$ is the prefix (partial solution). Let $\mathbf{P} = \mathbf{I} - \mathbf{M}$ be the orthogonal projection onto the unsolved (free) subspace. We decompose the state as $x = \mathbf{M}x + \mathbf{P}x$. Since $\mathbf{M}x$ is fixed to the constant vector $c$, the forward SDE (e.g., Variance Preserving or Variance Exploding) admits a degenerate transition kernel $q_t(x|c) = \nu_t(z|c) \otimes \delta_c(u)$, where $z = \mathbf{P}x \in \mathbb{R}^{d_{eff}}$ and $d_{eff} = \text{rank}(\mathbf{P})$.

The conditional score function $\nabla \log q_t(x|c)$ is only well-defined on the subspace $\text{Im}(\mathbf{P})$, and its components in $\text{Im}(\mathbf{M})$ vanish. Consequently, the reverse SDE is strictly confined to $\mathcal{V}_c$:
\begin{equation}
    dx = \{ f(x,t) - g(t)^2 \nabla \log q_t(x|c) \} dt + g(t) d\bar{w}
\end{equation}
where the stochastic evolution and score estimation error $\epsilon_{score}$ are restricted to $d_{eff}$ degrees of freedom, effectively reducing the ambient complexity from $d$ to $d_{eff}$.

% \noindent\textbf{2. Conditional Covariance Dissipation and KL Scaling.}
To bound the KL divergence on the manifold, we extend the nearly $d$-linear framework \citep{benton2023nearly} to the conditional case. Like Assumption 1 in \citep{benton2023nearly}, we assume the model achieves a per-dimension $L^2$ error $\epsilon_{score}$ specifically on the unsolved manifold $\mathcal{V}_c$.

A critical pillar of this proof is the Conditional Covariance Dissipation Identity. Let $\boldsymbol{\Sigma}_t := \text{Cov}(X_0 | X_t, c)$ be the conditional posterior covariance. Following the Tweedie relation, the score function's Hessian is related to $\boldsymbol{\Sigma}_t$. For a general class of linear SDEs (not limited to OU), the expected covariance evolves as:
\begin{equation}
    \frac{\sigma_t^3}{2\dot{\sigma}_t} \frac{d}{dt} \mathbb{E}[\boldsymbol{\Sigma}_t | c] = - \mathbb{E}[\boldsymbol{\Sigma}_t^2 | c]
\end{equation}
where $\sigma_t^2$ is the noise schedule. While \citep{benton2023nearly} derived this for the unconditional case, the identity holds on the subspace $\text{Im}(\mathbf{P})$ due to the orthogonality of $\mathbf{M}$ and $\mathbf{P}$. Since $\text{rank}(\boldsymbol{\Sigma}_t) = d_{eff}$, the trace $\text{Tr}(\boldsymbol{\Sigma}_t)$—which governs the discretization drift—scales linearly with $d_{eff}$. 

To avoid the score singularity at $t=0$, we consider the divergence between the generated distribution $p$ and the Gaussian-smoothed target $q_\delta$ at early-stopping time $\delta > 0$. Integrating over the time horizon $T$, the KL divergence scales as:
\begin{equation}
D_{KL}(q_\delta  \| p \mid c) \le \mathcal{C} [ \epsilon_{score}^2 + \kappa^2 d_{eff} N + \kappa d_{eff} T + d_{eff} e^{-2T} ] = \mathcal{O}(d_{eff} \cdot \epsilon_{score}^2)
\end{equation}
The initialization and discretization error components are fundamentally linear in the number of unsolved variables $d_{eff}$.

% \noindent\textbf{3. Mapping to Optimization Regret.}
Finally, we invoke Lemma \ref{lemma:kl_to_regret} (The Geometry-Statistics Bridge). Given the $\ell$-Lipschitz continuity of the objective $f$, the regret is bounded by the $W_1$ distance, which is in turn controlled by the KL divergence via the transportation-cost inequality. Substituting the $d_{eff}$-linear KL bound into the bridge inequality yields:
\begin{equation}
\mathcal{R}_{DM}(p, q \mid c) \le \ell \cdot \mathcal{C} \sqrt{D_{KL}(q_\delta \| p \mid c)} \le \mathcal{O}(\ell \cdot \sqrt{d_{eff}} \cdot \epsilon_{score})
\end{equation}
where we treat the discrepancy between $q$ and $q_\delta$ as a negligible additive bias for $\delta \ll 1$.
\end{proof}

\subsection{Justification of Lemma \ref{lemma:rl_scaling}: Super-linear RL Horizon Scaling}
\label{app:rl_scaling}

Lemma~\ref{lemma:rl_scaling} models the accumulated regret of an autoregressive RL construction policy as super-linear in the decision horizon. This scaling is motivated by standard regret bounds in finite-horizon RL. In tabular episodic MDPs, Q-learning with UCB exploration achieves a regret bound of the form
\[
    \widetilde{O}\!\left(\sqrt{H^3 S A T}\right),
\]
where \(H\) is the episode horizon, \(S\) and \(A\) are the numbers of states and actions, and \(T\) is the total number of interaction steps~\citep{jin2018Qlearning}. For RL with linear function approximation, related analyses obtain regret bounds of the form
\[
    \widetilde{O}\!\left(\sqrt{d^3 H^3 T}\right),
\]
where \(d\) denotes the feature dimension~\cite{jin2020provably,velegkas2022reinforcement}. In both cases, the dependence on the horizon contains an \(H^{3/2}\) factor, which is super-linear in \(H\). 
More broadly, super-linear horizon dependence is a common feature of finite-horizon RL regret analyses. 
These bounds provide theoretical support for the horizon-scaling model in Lemma~\ref{lemma:rl_scaling}, although they do not directly characterize neural CO solvers.

Autoregressive neural CO solvers such as AM and POMO instantiate sequential construction policies: a solution is produced through a sequence of decisions whose horizon grows with the problem size. Existing finite-horizon RL bounds do not directly characterize these neural policy-gradient solvers, but they provide a principled motivation for modeling accumulated RL regret as super-linear in the construction horizon. This modeling choice is also consistent with the empirical behavior of autoregressive RL solvers on large-scale CO instances, where solution quality often degrades as the number of construction steps increases.

\subsection{Proof and Empirical Regret Analysis of Theorem \ref{thm:optimal_handoff}: Existence and Uniqueness of the Optimal Trigger Step $\tau^*$}
\label{sec:app_proof_optimal_trigger}
\begin{proof}
This proof analyzes the hybrid regret $\mathcal{L}(\tau)$ to establish the uniqueness and global optimality of the trigger step $\tau^*$. 

Based on the definition of $\mathcal{L}(\tau) = \int_0^{\tau} r_{\text{RL}}(h)\,dh + \int_{\tau}^{d} r_{\text{DM}}(h)\,dh$, we can get that the derivative of the hybrid regret with respect to the trigger step is:
$$\mathcal{L}'(\tau) = r_{\text{RL}}(\tau) - r_{\text{DM}}(\tau).$$

According to Lemma \ref{lemma:rl_scaling}, $r_{RL}(h) = \frac{d}{dh}\mathcal{R}_{\text{RL}}(h) \sim \mathcal{O}(h^{\alpha-1})$ with $\alpha -1 >0$, which is strictly increasing in $h$, i.e. $r_{\text{RL}}(h) \geq 0$.

For the regret of DM, it satisfies $R_{\text{DM}}(d) \leq \mathcal{O}(\sqrt{d})$ by Lemma \ref{thm:dm_scaling}, where $R_{\text{DM}}(d)$ means the regret of DM (fixed distribution $p$ and $q$ so we ignore its signal in the regret) with $d$-dimensional decision variables. Since $\sqrt{d}$ is concave in $d$, $R_{\text{DM}}(d)$ exhibits diminishing marginal returns as the problem size $d$ increases. In the hybrid framework, the regret of conditional DM facing a problem of size $d-\tau$ is calculated by $\int_{\tau}^{d} r_{\text{DM}}(h)\,dh = \int_{0}^{d} r_{\text{DM}}(h)\,dh - \int_{0}^{\tau} r_{\text{DM}}(h)\,dh = R_{\text{DM}}(d) - R_{\text{DM}}(\tau)$. Due to the decreasing marginal regret of $R_{\text{DM}}$, we can get $r_{\text{DM}}(\tau)$ is non-increasing. 

Since $r_{\text{RL}}(\tau)$ is strictly increasing and $r_{\text{DM}}(\tau)$ is non-increasing, their difference $\mathcal{L}'(\tau) = r_{\text{RL}}(\tau) - r_{\text{DM}}(\tau)$ is strictly increasing.
The boundary conditions $r_{\text{RL}}(0) < r_{\text{DM}}(0)$ and $r_{\text{RL}}(d) > r_{\text{DM}}(d)$ imply $\mathcal{L}'(0) < 0$ and $\mathcal{L}'(d) > 0$. Since $\mathcal{L}'(\tau)$ is continuous and strictly increasing, the Intermediate Value Theorem guarantees exactly one $\tau^* \in (0, d)$ such that $\mathcal{L}'(\tau^*) = 0$, i.e., $r_{\text{RL}}(\tau^*) = r_{\text{DM}}(\tau^*)$. Since $\mathcal{L}'(\tau) < 0$ for $\tau < \tau^*$ and $\mathcal{L}'(\tau) > 0$ for $\tau > \tau^*$, this $\tau^*$ is the unique global minimum of $\mathcal{L}(\tau)$.

Specially, since $\mathcal{L}'(\tau)$ is strictly increasing with $\mathcal{L}'(0) < 0$ and $\mathcal{L}'(d) > 0$, $\mathcal{L}$ is strictly decreasing on $[0, \tau^*)$ and strictly increasing on $(\tau^*, d]$. This implies $\mathcal{L}(\tau^*) < \mathcal{L}(0)$ and $\mathcal{L}(\tau^*) < \mathcal{L}(d)$, i.e., the hybrid framework achieves strictly lower expected regret than either pure RL or pure DM alone.
\end{proof}

\paragraph{Generality of the sufficient conditions.}
The proof above establishes the conclusion through the
continuity, strict monotonicity, and boundary sign change
of the marginal-regret difference $\mathcal{L}'(\tau)$.
The particular scaling forms considered in
Lemmas~\ref{thm:dm_scaling} and~\ref{lemma:rl_scaling}
are not necessary for the conclusion.
In particular, $\mathcal{L}'(\tau)$ remains strictly
increasing when $r_{\mathrm{RL}}(\tau)$ is non-decreasing
and $r_{\mathrm{DM}}(\tau)$ is strictly decreasing.
This permits regret profiles beyond the specific scaling
forms considered above, including linear RL regret
accumulation, provided that the stated boundary conditions
also hold. The applicability of the result is conditional
on these structural properties of the regret profiles.

\paragraph{Empirical examination of the regret curves.}
We examine the regret profiles of the trained solvers on
128 TSP-500 instances. For the RL POMO solver, we greedily
construct prefixes of different lengths $h$ and solve each
remaining prefix-constrained subproblem to optimality.
We estimate the accumulated regret as the difference
between the optimal objective under the constructed prefix
and the unconstrained optimal objective, and fit a
power-law model $R_{\mathrm{RL}}(h)\approx a h^\alpha$
to the measurements for each instance.
Fitting $R_{\mathrm{RL}}(h)\propto h^\alpha$ yields
per-instance exponents ranging from $0.93$ to $1.16$,
with a median of $\hat{\alpha}=1.02$.
These results indicate approximately linear regret
accumulation over the measured range, with variation
across individual instances.

For the DM solver, we condition the trained model on progressively longer prefixes taken from an optimal tour.
This controls prefix quality and isolates the effect of the remaining free dimension $d_{\mathrm{eff}}$. We measure the completion regret relative to the optimal tour and fit an exponential model to its dependence on $d_{\mathrm{eff}}$. Over the measured range, the completion regret increases approximately exponentially with the remaining dimension,
$
R_{\mathrm{DM}}(d_{\mathrm{eff}})
\propto \exp(\lambda d_{\mathrm{eff}}), \lambda\approx0.019.
$
Equivalently, the fitted completion regret decreases approximately exponentially as the prefix grows and the remaining subspace shrinks. This finite-range empirical fit supports the monotonic dimension dependence predicted by Lemma ~\ref{thm:dm_scaling}, but it is not intended as a tight characterization of the theoretical upper bound. Importantly, these empirical forms still preserve the relative monotonicity required by Theorem \ref{thm:optimal_handoff}: $\mathcal{L}'(\tau)$ is strictly increasing.

Overall, Theorem 4.6 characterizes the unique-handover result under explicit conditions. Lemmas 4.4 and 4.5 provide one sufficient route to its conclusion through specific RL and DM scaling forms, whereas the proof itself relies on the more general relative one-crossing condition. The empirical curves exhibit an alternative instantiation of the required structure: an approximately linear RL regret curve with a non-decreasing marginal trend, together with an exponentially decreasing conditional-DM marginal regret. For the fitted aggregate trends, this makes 
$\mathcal{L}'(\tau)$ increasing, which is consistent with the one-crossing condition and supports the same unique-handover mechanism under the stated boundary conditions.

\section{Additional Experimental Results}
\label{sec:app_exp_results}

For the main performance comparisons and the following result tables in Appendix \ref{sec:app_exp_results}, we evaluate on 1,280 test instances each for TSP-50 and TSP-100, and 128 instances each for TSP-500 and TSP-1000. For MIS, the RB-LARGE, ER-700–800, and SATLIB test sets contain 500, 128, and 500 instances, respectively. For OP, we evaluate on 500 test instances at each problem size: OP-50, OP-100, and OP-200. These counts refer to distinct problem instances and are separate from the number of candidate solutions generated per instance.

\subsection{TSP Performance Results}
\label{sec:app_exp_results_tsp}

For TSP-50/100, we omit wall-clock time because several baselines are quoted from prior work under heterogeneous hardware settings, while greedy inference for small instances is already in the millisecond-scale regime and can be dominated by hardware, batching, and implementation overhead. We therefore focus time reporting on large-scale TSP, where computational differences are more meaningful.

\begin{table*}[htbp]
\centering
%\scriptsize
\small
\setlength{\tabcolsep}{12pt}
\caption{Results with \textbf{Greedy Decoding} on TSP-50 and TSP-100. RL: Reinforcement Learning, SL: Supervised Learning, G: Greedy Decoding, S: Sample Decoding. $^*$ denotes results that are quoted from previous works. The ``Prefix\_Difusco (S)`` represents the best result obtained by running our Prefix\_Difusco model sampling 16 times in parallel.}
\label{tab_app:tsp100_results}
\setlength{\tabcolsep}{10pt} % Adjust column spacing
\renewcommand{\arraystretch}{1.1} % Adjust row height
\begin{tabular}{llcccc}
\toprule
\textbf{ALGORITHM} & \textbf{TYPE} & \multicolumn{2}{c}{\textbf{TSP-50}} & \multicolumn{2}{c}{\textbf{TSP-100}} \\
\cmidrule(lr){3-4} \cmidrule(lr){5-6}
& & Len $\downarrow$ & Gap $\downarrow$ & Len $\downarrow$ & Gap $\downarrow$ \\
\midrule
Concorde \citep{Applegate2006} & Exact & 5.69 & 0.00\% & 7.76 & 0.00\% \\
2Opt \citep{Croes1958} & Heuristics & 5.86 & 2.95\% & 8.03 & 3.54\% \\
\midrule
AM* \cite{Kool2019} & RL+G & 5.80 & 1.76\% & 8.12 & 4.53\% \\
GCN* \cite{Joshi2019} & SL+G & 5.87 & 3.10\% & 8.41 & 8.38\% \\
Transformer* \cite{Bresson2021} & RL+G & 5.71 & 0.31\% & 7.88 & 1.42\% \\
POMO* \cite{Kwon2020} & RL+G & 5.73 & 0.64\% & 7.84 & 1.07\% \\
Sym-NCO* \cite{Kim2022} & RL+G & 5.73 & 0.64\% & 7.84 & 0.94\% \\
Image Diffusion* \cite{Graikos2022} & SL+G & 5.76 & 1.23\% & 7.92 & 2.11\% \\
DIFUSCO ($T_s$=50) \cite{Sun2023} & SL+G & 5.71 & 0.45\% & 7.85 & 1.21\% \\
%DIFUSCO ($T_s$=100) & SL+G & 5.71 & 0.41\% & 7.84 & 1.16\% \\
T2T ($T_s$=50,$T_t$=15) \cite{li2023t2t} & SL+G & 5.69 & 0.07\% & 7.77 & 0.20\% \\
\textbf{Prefix\_Difusco} & SL+G & 5.70 & 0.27\% & 7.81 & 0.62\% \\
\textbf{Prefix\_Difusco} (S) & SL+S & 5.69 & 0.04\% & 7.76 & 0.04\% \\
\textbf{HyCO (AM)} & \textbf{RL+SL+G} & 5.69 & \textbf{0.03\%} & 7.76 & \textbf{0.02\%} \\
\textbf{HyCO (POMO)} & \textbf{RL+SL+G} & 5.69 & \textbf{0.03}\% & 7.76 & \textbf{0.02\%} \\
\midrule
AM & RL+G+2OPT & 5.77 & 1.41\% & 8.02 & 3.32\% \\
GCN & SL+G+2OPT & 5.77 & 1.40\% & 8.01 & 3.21\% \\
Transformer & RL+G+2OPT & 5.70 & 1.06\% & 7.96 & 1.89\% \\
POMO & SL+G+2OPT & 5.73 & 0.63\% & 7.91 & 1.62\% \\
Sym-NCO & SL+G+2OPT & 5.73 & 0.64\% & 7.90 & 0.76\% \\
DIFUSCO & SL+G+2OPT & 5.69 & 0.09\% & 7.78 & 0.22\% \\
T2T & SL+G+2OPT & 5.69 & \textbf{0.02\%} & 7.76 & 0.06\% \\
\textbf{HyCO (POMO)} & \textbf{RL+SL+G+2OPT } & 5.69 & 0.03\% & 7.76 & \textbf{0.01\%} \\
\bottomrule
\end{tabular}
\end{table*}

\noindent\textbf{Impact of Local Search (2-opt).}
To investigate whether the improvements of HyCO stem merely from local refinements or from a fundamentally better global solution structure, we evaluated the performance when a 2-opt local search post-processing step \citep{Croes1958} is applied. As indicated in our experimental logs, while 2-opt consistently improves the objective values for all constructive methods, the relative advantage of HyCO persists. On large-scale instances (TSP-1000), even after exhaustive 2-opt refinement, the standard RL baselines (e.g., POMO+2opt) typically fail to reach the solution quality of HyCO. This observation implies that the baseline solvers often get trapped in poor local basins due to early-stage compounding errors, which local search cannot rectify. In contrast, HyCO utilizes the diffusion model to bridge the critical regime shift, producing a globally superior solution skeleton that is more amenable to local refinement. This confirms that HyCO contributes to \textit{global consistency} rather than just local smoothness.

\noindent\textbf{Sample Efficiency and Inference Budget.}
A critical comparison in our results lies between HyCO and the brute-force conditional baseline, Prefix\_Difusco (S). The latter represents a standard sampling approach where the model generates 16 independent completions for a fixed prefix and selects the best one. Remarkably, HyCO achieves solution quality that is comparable to or better than this baseline despite using a single guided completion ($S=4$). For instance, on TSP-100, HyCO (0.02\% gap) outperforms the 16-sample baseline (0.04\% gap). This result reinforces the core premise of our theory: the \textit{triggering step} is more consequential than the \textit{quantity} of samples. By adaptively identifying the optimal triggering time window, HyCO eliminates the need for wasteful over-sampling, achieving superior efficiency with \textbf{4$\times$ fewer inference passes}. This structural advantage suggests that HyCO effectively locates the optimal low-regret subspace, whereas blind over-sampling struggles to compensate for a suboptimal triggering time.

\noindent\textbf{Plug-and-Play Versatility.}
To validate the universality of our framework, we extended our evaluation by pairing HyCO with three distinct sequential architectures: the standard autoregressive AM, the symmetry-enhanced POMO, and the divide-and-conquer LEHD. HyCO yields consistent performance gains across all backbones, regardless of their intrinsic strength. For weaker solvers like AM, HyCO acts as a revitalizing mechanism, reducing the optimality gap on TSP-1000 from over 20\% to about 5\% by correcting severe horizon-wise error accumulation. Conversely, for state-of-the-art constructive solvers like LEHD, HyCO provides further refinement, pushing the limits of solution quality. This consistent improvement confirms that HyCO serves as a general-purpose, test-time inference wrapper that addresses the fundamental uncertainty accumulation inherent to sequential construction, rather than being tailored to the idiosyncrasies of a specific policy.

\begin{table}[h!]
\centering
\small
%\scriptsize
\setlength{\tabcolsep}{4pt} % 适当减小列间距
\caption{Results on large-scale TSP problems. RL, SL, AS, G, and S denote Reinforcement Learning, Supervised Learning, Active Search, Greedy decoding, and Sample decoding, respectively. %\textcolor{red}{POMO sampling 7680 and 3840 for TSP-500 and 1000 separately.} 
Len means the average tour length. * indicates the baseline for computing the performance gap. The Time column reports the average inference time per instance.} % 根据需要修改标题
\label{tab:app_tsp_large_results}
\renewcommand{\arraystretch}{1.10} % 增加行高，使表格更易读

\begin{tabular}{llcccccc}
\toprule
\centering

%\scriptsize
\setlength{\tabcolsep}{4pt} % 适当减小列间距
\textbf{ALGORITHM} & \textbf{TYPE} & \multicolumn{3}{c}{\textbf{TSP-500}} & \multicolumn{3}{c}{\textbf{TSP-1000}} \\
\cmidrule(lr){3-5} \cmidrule(lr){6-8}
& & Len$\downarrow$ & Gap$\downarrow$ & Time$\downarrow$ & Len$\downarrow$ & Gap$\downarrow$ & Time$\downarrow$ \\
\midrule
Concorde        & Exact        & 16.55* & 0.00\%   & 37.66m & 23.53* & 0.00\%   & 6.65h  \\
Gurobi          & Exact        & 16.55  & 0.00\%   & 45.63h & 23.53  & 0.00\%   & 48h   \\
LKH-3 (default) & Heuristics   & 16.55  & 0.00\%   & 46.28m & 23.53  & 0.00\%   & 2.57h  \\
\midrule
AM              & RL+G         & 20.02  & 20.99\%  & \textbf{0.47s}  & 28.52  & 21.21\%  & 1.09s \\
GCN             & SL+G         & 29.72  & 79.61\%  & 6.67m  & 43.15  & 83.38\%  & 28.52m  \\
POMO     & RL+G      & 19.13  & 15.62\%  & 0.49s & 27.15   & 15.38\%        & \textbf{0.94s}     \\
POMO(S)     & RL+S      &  18.59  & 12.32\%  & 2.97m  &   28.34    &  20.44\%      &  3.25m   
\\
DIMES           & RL+G         & 18.93  & 14.38\%  & 0.97m  & 27.23  & 15.73\%  & 2.08m  \\
DIMES           & RL+AS+G      & 17.81  & 7.61\%   & 2.10h  & 25.11  & 6.72\%   & 4.49h  \\
LEHD           & RL/SL+G    &   16.82  & 1.64\%  &  1.82s  &  24.29 & 3.23\%   &   3.84s\\
DIFUSCO         & SL+G         & 18.11  & 9.41\%   & 5.70m  & 25.68  & 9.14\%   & 11.5m  \\
T2T             & SL+G         & 17.39  & 5.09\%   & 4.90m  &25.17  & 8.87\%   & 15.66m  \\
Prefix\_Difusco  & SL+G         & 17.92  & 8.23\%   & 0.14m  & 25.86  & 9.91\%   & 0.35m  \\
Prefix\_Difusco(16 samples)  & SL+S         & 17.31  & 4.58\%   & 2.10m  & 24.96  & 6.07\%   & 5.94m  \\
\textbf{HyCO(POMO)}           & \textbf{RL+SL+G} & 17.24 & 3.56\% & 1.20m & 24.61 & 4.58\% & 2.43m \\
\textbf{HyCO(LEHD)}           & \textbf{RL+SL+G} & \textbf{16.81} & \textbf{1.57\%} & 0.35m & \textbf{24.24} & \textbf{3.04\%} & 0.36m \\
\midrule
POMO     & RL+G+2OPT      & 17.77  & 7.37\%  & 0.58s & 25.10   & 6.67\%        & 1.05s     \\
DIMES              & RL+G+2OPT      & 17.65  & 6.62\%  & 1.01m & 24.83  & 7.38\%  & 2.29m \\ 
DIMES              & RL+AS+G+2OPT   & 17.31  & 4.57\%  & 2.10h & 24.33  & 5.22\%  & 4.49h \\
LEHD           & RL/SL+G+2OPT    &   16.77  & 1.40\%  &  1.83s  & 24.06& 2.29\%   &   3.85s\\
DIFUSCO                & SL+G+2OPT        & 16.83  & 1.68\%  & 5.75m  &23.92  & 1.66\%  & 17.52m \\
T2T                    & SL+G+2OPT        &16.86  & 1.92\%  & 2.42m  &23.86  &1.42\%  & 15.90m \\
%Prefix\_Difusco(used in HyCO)  & SL+G+2OPT         & 16.83  & 1.68\%   & 0.56m  & 23.93  & 1.72\%   & 0.52m  \\
\textbf{HyCO(POMO)}           & \textbf{RL+SL+G+2OPT} & 16.82 & 1.63\% & 2.32m & 23.92 & 1.66\% & 4.40m \\
\textbf{HyCO(LEHD)}           & \textbf{RL+SL+G+2OPT} & \textbf{16.74} & \textbf{1.2\%} & 0.36m & \textbf{23.64} & \textbf{0.46\%} & 0.36m \\
\bottomrule
\end{tabular}
\end{table}

\subsection{MIS Performance Results}
\label{sec:app_exp_results_mis}

To validate the versatility of HyCO beyond routing tasks, we extended our evaluation to the MIS problem, a fundamental graph covering challenge. It is crucial to clarify the definition of the sampling budget $S$ to appreciate the efficiency gains fully. For the baseline, $S$ denotes the number of independent trajectories generated from scratch; for HyCO, $S$ represents the number of top-ranked RL candidates selected at the trigger step for diffusion completion. As detailed in Table~\ref{tab:app_mis_results}, we benchmarked HyCO against the specialized RL baseline (LWD) and the Diffusion backbone (prefix\_COExpander) under an identical inference budget ($S=5$). This setup allows us to rigorously assess whether the performance gains stem from our adaptive trigger and hybrid mechanism rather than simply increased computational resources.

The results demonstrate a strong synergy similar to our findings in TSP. HyCO consistently outperforms both learning-based baselines across all datasets, with the advantage being most pronounced on the challenging ER-700-800 benchmark. While the standalone RL policy suffers a significant 13.16\% optimality gap and the pure Diffusion model plateaus at 9.16\%, HyCO effectively combines their strengths to achieve a much lower gap of \textbf{5.76\%}. Similarly, on RB-Large, HyCO reduces the RL error by nearly two-thirds (from 15.31\% to 5.28\%).

Notably, HyCO ($S=5$) not only surpasses the baseline at the same budget but also remains competitive with or superior to baselines using significantly larger budgets (e.g., $S=64$, as observed in typical diffusion evaluations). This indicates that the performance gain is attributable to the \textit{quality} of the trigger-guided subspace exploration rather than the \textit{quantity} of blind sampling, a phenomenon consistent with our observations in the TSP experiments.

% \begin{table*}[htbp!]
% \centering
% \caption{Performance comparison across RB-LARGE, ER-700-800, and SATLIB benchmarks for the Maximum Independent Set (MIS) problem. Obj. denotes the average node number of independent sets (higher is better). Gap indicates the percentage deviation from the optimal solution (lower is better).}
% \label{tab:mis_results}
% %\footnotesize
% \setlength{\tabcolsep}{8pt}  
% \begin{tabular}{l l r r  r r  r r }
% \toprule
% \textbf{METHOD} & \textbf{TYPE} & \multicolumn{2}{c}{\textbf{RB-LARGE}} & \multicolumn{2}{c}{\textbf{ER-700-800}} & \multicolumn{2}{c}{\textbf{SATLIB}} \\
% \cmidrule(lr){3-4} \cmidrule(lr){5-6} \cmidrule(lr){7-8} 
% & & Obj.$\uparrow$ & Gap$\downarrow$ & Obj.$\uparrow$ & Gap$\downarrow$ & Obj.$\uparrow$ & Gap$\downarrow$ \\
% \midrule
% KaMIS & Exact & 43.00 & 0.00\%  & 44.97 & 0.00\%  & 425.95 & 0.00\% \\
% Gurobi & Heuristics & 42.19 & 1.83\%  & 38.78 & 13.75\%  & 425.92 & 0.01\% \\
% \midrule
% LWD & RL+G & 36.67  & 15.31\% & 39.05 & 13.16\% & 421.80 & 1.04\%  \\
% prefix\_Coexpander & SL+G &  40.05& 6.84\% & 40.18 & 10.65\% & 421.84 & 0.96\% \\
% prefix\_Coexpander (s=5) & SL+S & 40.06 & 6.83\% & 40.85 & 9.16\% & 424.17 & 0.42\% \\
% prefix\_Coexpander(s=64)& SL+S & 40.58 & 5.63\% & 40.91 & 9.03\% &  424.78 & 0.27\% \\
% \textbf{HyCO} ($S_{cand}$=5) & RL+SL+G & \textbf{40.73} & \textbf{5.28\%} & \textbf{42.38} & \textbf{5.76\%} &  \textbf{425.40} & \textbf{ 0.12\%} \\
% \bottomrule
% \end{tabular}
% \setlength{\tabcolsep}{6pt} % 恢复默认列间距
% \end{table*}

\begin{table*}[htbp!]
\centering
\caption{Performance comparison across RB-LARGE, ER-700-800, and SATLIB benchmarks for the Maximum Independent Set (MIS) problem. Obj. denotes the average node number of independent sets (higher is better). Gap indicates the percentage deviation from the optimal solution (lower is better). Time is the average inference time in seconds. S: Sample Decoding. $S_{cand}$ is the candidate number of full DM proposals $\mathcal{P}_{DM}$.}
\label{tab:app_mis_results}
%\scriptsize               % 保持小字体
%\footnotesize
\small
\setlength{\tabcolsep}{4pt}  
\begin{tabular}{l r r r r r r r r r}
\toprule
\textbf{METHOD} & \multicolumn{3}{c}{\textbf{RB-LARGE}} & \multicolumn{3}{c}{\textbf{ER-700-800}} & \multicolumn{3}{c}{\textbf{SATLIB}} \\
\cmidrule(lr){2-4} \cmidrule(lr){5-7} \cmidrule(lr){8-10} 
 & Obj.$\uparrow$ & Gap$\downarrow$ & Time$\downarrow$ & Obj.$\uparrow$ & Gap$\downarrow$ & Time$\downarrow$ & Obj.$\uparrow$ & Gap$\downarrow$ & Time$\downarrow$ \\
\midrule
KaMIS & 43.00 & 0.00\% & 56.97s & 44.97 & 0.00\% & 60.75s & 425.95 & 0.00\% & 24.37s \\
Gurobi & 42.19 & 1.83\% & 33.84s & 38.78 & 13.75\% & 60.49s & 425.92 & 0.01\% & 13.47s \\
\midrule
LWD (RL) & 36.67  & 15.31\% &\textbf{0.85s} & 39.05 & 13.16\% & \textbf{ 0.64s }& 421.80 & 1.04\%  & 0.66s \\
prefix\_Coexpander (Greedy) &  40.05& 6.84\% & 1.15s & 40.18 & 10.65\% & 0.71s & 421.84 & 0.96\% & \textbf{0.25s} \\
prefix\_Coexpander (s=5) & 40.06 & 6.83\% & 1.22s & 40.85 & 9.16\% & 0.96s & 424.17 & 0.42\% & 0.91s \\
prefix\_Coexpander(s=64)  & 40.58 & 5.63\% & 4.48s & 40.91 & 9.03\% & 4.03s &   424.78& 0.27\% & 4.27s \\
\textbf{HyCO} ($S_{cand}$=5) & \textbf{40.73} & \textbf{5.28\%} & 2.69s & \textbf{42.38} & \textbf{5.76\%} &  4.69s &  \textbf{425.40} & \textbf{ 0.12\%} & 6.31s \\
\bottomrule
\end{tabular}
\setlength{\tabcolsep}{6pt} % 恢复默认列间距
\end{table*}

\subsection{OP Performance Results}
\label{sec:app_exp_results_op}

While our primary analysis focuses on TSP and MIS, we further evaluate HyCO on the Orienteering Problem (OP) to assess its generality on a different routing problem. 
We consider OP-50, OP-100, and OP-200, and compare HyCO against both the standalone RL backbone AM and the diffusion backbone prefix\_Coexpander under different sampling budgets, as summarized in Table~\ref{tab:op_comprehensive_results}. 
Across all scales, HyCO consistently achieves the best performance among learning-based methods, improving over both the RL and diffusion baselines.

The advantage becomes more pronounced as the problem size increases. 
On OP-200, AM obtains an optimality gap of 13.05\%, while prefix\_Coexpander with the same sampling budget \(s=8\) reduces the gap to 4.15\%. 
HyCO further lowers the gap to \textbf{1.08\%}, corresponding to a 91.7\% relative reduction over AM, and a 74.0\% relative reduction over prefix\_Coexpander \((s=8)\). 
Notably, HyCO also outperforms the much stronger prefix\_Coexpander \((s=64)\) baseline, which attains a gap of \textbf{2.77\%} on OP-200. 
This suggests that the improvement is not merely due to increasing the number of diffusion samples, but rather comes from using an informative RL prefix and an adaptive handover to guide diffusion generation toward a more favorable subspace.

HyCO achieves these gains with moderate computational overhead. 
For OP-200, HyCO takes \textbf{2.70s} per instance, which is comparable to prefix\_Coexpander \((s=64)\) and substantially faster than the Gurobi-300 reference. 
These results demonstrate that HyCO can serve as a plug-and-play hybrid inference mechanism beyond TSP, effectively combining weaker individual components into a stronger solver through adaptive RL-to-DM handover.

\begin{table}[htbp!]
\centering
\caption{Performance comparison across OP-50, OP-100, and OP-200 benchmarks. Obj. denotes the average collected prize (higher is better). Gap indicates the percentage deviation from the optimal solution (lower is better). Time is the average inference time in seconds. S: Sample Decoding. $S_{cand}$ is the candidate number of full DM proposals $\mathcal{P}_{DM}$.}
\label{tab:op_comprehensive_results}
\renewcommand{\arraystretch}{1.1} % 稍微增加行高，提升可读性
\setlength{\tabcolsep}{3pt}    
%\footnotesize   
\small
\begin{tabular}{l *{3}{ccc}} 
\toprule
\textbf{Method} & \multicolumn{3}{c}{\textbf{OP-50}} & \multicolumn{3}{c}{\textbf{OP-100}} & \multicolumn{3}{c}{\textbf{OP-200}} \\
\cmidrule(lr){2-4} \cmidrule(lr){5-7} \cmidrule(lr){8-10}
& Obj. $\uparrow$ & Gap $\downarrow$ & Time $\downarrow$
& Obj. $\uparrow$ & Gap $\downarrow$ & Time $\downarrow$
& Obj. $\uparrow$ & Gap $\downarrow$ & Time $\downarrow$ \\
\midrule
% --- Solver Baselines ---
Gurobi-300      & 14.37 & 0.00\%    & 300.0s & 32.10 & 0.00\%   & 300.0s & 44.38 & 0.00\%   & 300.0s \\
Gurobi-30       & 14.31 & 0.40\%    & 30.0s  & 31.13 & 3.02\%   & 30.0s  & 31.37 & 29.31\%  & 30.0s \\
\midrule
AM (RL)                     & 11.92 & 17.04\% & \textbf{0.03s} & 28.37 & 11.62\% & \textbf{0.03s} & 38.59 & 13.05\% & \textbf{0.09s} \\
prefix\_Coexpander (Greedy) & 12.27 & 14.61\% & 0.09s & 28.64 & 10.77\% & 0.15s & 41.36 & 6.80\%  & 0.19s \\
prefix\_Coexpander (s=8)    & 12.42 & 14.01\% & 0.12s & 29.05 & 9.49\%  & 0.21s & 42.54 & 4.15\%  & 0.47s \\
prefix\_Coexpander (s=64)   & 12.51 & 12.94\% & 1.08s & 29.20 & 9.03\%  & 0.89s & 43.15 & 2.77\%  & 2.67s \\
\textbf{HyCO} ($S_{cand}$=8)& \textbf{13.23} & \textbf{7.93\%} & 0.55s & \textbf{29.51} & \textbf{8.06\%} & 0.41s & \textbf{43.90} & \textbf{1.08\%} & 2.70s \\
\bottomrule
\end{tabular}
\end{table}

\subsection{Adaptive Trigger vs. Fixed-time Triggering}
\label{sec:app_exp_fixed_vs_adaptive}

To validate the effectiveness of our adaptive trigger, we compare HyCO against a fixed-time trigger baseline.
Specifically, we sweep over fixed triggering steps $k$ on a validation split and report the best-performing schedule (\textit{Best Fixed}).
This provides a strong reference for quantifying the benefit of adaptive triggering. The fixed-time baseline retains the same HyCO inference pipeline, and replaces only the adaptive trigger with a validation-selected fixed handover step.

Figure~\ref{fig:fixed_vs_adaptive_bar_horiz} reports the percentage optimality gap to the exact solver's solution for both TSP and MIS.
Across all datasets, HyCO with adaptive triggering consistently outperforms the best fixed-time baseline.
On TSP benchmarks, adaptive triggering yields improvements over a validation-selected schedule, reducing the gap from $0.026\%$ to $0.020\%$ on TSP-100 and from $5.66\%$ to $3.56\%$ on TSP-500.

A similar trend holds on MIS benchmarks.
On SATLIB, adaptive triggering reduces the gap from $1.20\%$ to $0.12\%$ ($1.08\%$ absolute improvement), bringing HyCO close to the exact reference.
On ER-700--800, where instances vary significantly due to random graph generation, adaptive triggering still improves over the best fixed schedule by $3.22\%$.
Overall, these results support that HyCO's trigger functions as a statistically effective triggering proxy: it reliably identifies a regime where conditional global completion provides net gains, without requiring any fixed schedule tuned per dataset or attempting to approximate the expectation-level optimal trigger step on each instance.

\begin{figure}[htbp]
\centering
\begin{tikzpicture}
    \begin{axis}[
        xbar,                       % 设置为横向柱状图
        bar width=12pt,             % 调整柱子粗细，双栏不宜太粗
        width=0.8\columnwidth,         % 【关键】宽度自动适应双栏的单栏宽度
        height=6.5cm,               % 高度适中
        xlabel={Optimality Gap (\%)}, % X轴标签（现在数值在X轴）
        % Y轴设置
        symbolic y coords={TSP-100, TSP-500, SATLIB, ER-700--800}, % 坐标名称
        ytick=data,                 % 显示所有Y轴刻度
        y dir=reverse,              % 【关键】反转Y轴，让TSP-100(表格第一行)显示在最上方
        ytick style={draw=none},    % 隐藏Y轴上的小刻度线，更干净
        y axis line style={draw=none}, % 隐藏Y轴那根竖线
        % X轴设置
        xmin=0, xmax=11.5,          % 设置X轴范围，留出空间给右边的数字标签
        xmajorgrids=true,           % 显示垂直网格线
        grid style={gray!10},       % 网格线颜色做淡一点
        axis x line*=bottom,        % 只显示底部的X轴线，隐藏顶部的
        xticklabel style={font=\footnotesize},
        % 柱子上的数值标签设置
        nodes near coords,          % 显示数值
        nodes near coords align={horizontal}, % 数值水平排列
        nodes near coords style={
            font=\scriptsize\bfseries, % 字体更小且加粗
            color=black!80,            % 颜色不要纯黑，柔和一点
            anchor=west,               % 数字对齐在柱子右侧
            /pgf/number format/fixed,
            /pgf/number format/precision=2 % 保留两位小数
        },
        % 图例设置
        legend style={
            at={(0.5,1)},        % 放在图表上方居中
            anchor=south,
            legend columns=-1,      % 图例横向排列
            draw=none,              % 去掉图例边框
            fill=none,              % 去掉图例背景
            /tikz/every even column/.append style={column sep=0.1cm} % 图例间距
        },
        enlarge y limits=0.15,      % 增加Y轴上下的留白
    ]

    % 数据组 1: Best Fixed
    \addplot[
        fill=academicBlue,          % 使用自定义蓝色
        draw=none                   % 去掉柱子边框
    ] coordinates {
        (0.026,TSP-100) 
        (5.66,TSP-500) 
        (1.20,SATLIB) 
        (8.98,ER-700--800)
    };

    % 数据组 2: HyCO (Adaptive)
    \addplot[
        fill=academicRed,           % 使用自定义红色
        draw=none
    ] coordinates {
        (0.020,TSP-100) 
        (3.56,TSP-500) 
        (0.12,SATLIB) 
        (5.76,ER-700--800)
    };

    \legend{Best Fixed, HyCO (Adaptive)}
    \end{axis}
\end{tikzpicture}
\caption{Adaptive trigger vs. best fixed-time trigger across tasks.
Best Fixed is selected by validation sweeping over fixed triggering steps.
We report the percentage optimality gap to the
exact solver’s solution for TSP and MIS. Lower is better.}
\label{fig:fixed_vs_adaptive_bar_horiz}
\vspace{-10pt}
\end{figure}
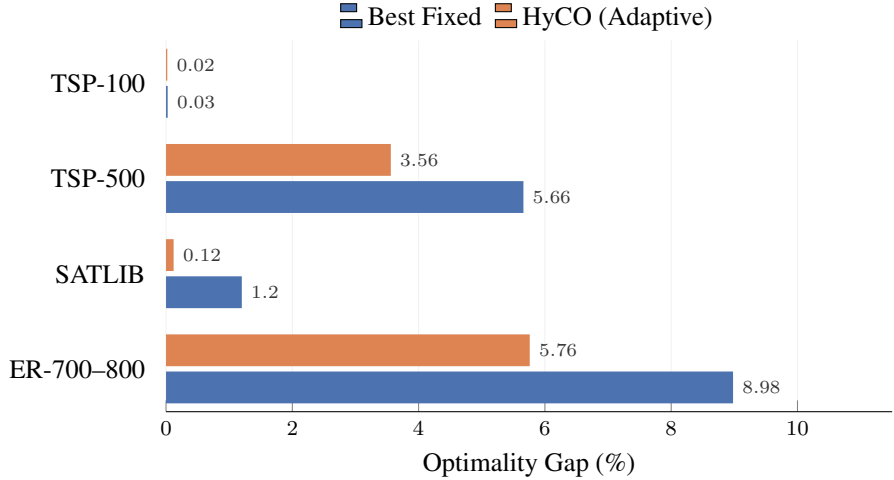

\subsection{Cross-Dataset Trigger Statistics}
\label{sec:app_trigger_stats}

Here we report the average trigger step across different problem classes and scales in Table~\ref{tab:trigger_stats}. Under the experimental settings used in this paper, the handover timing varies substantially across tasks and scales rather than following a fixed early-trigger pattern. While TSP-100 tends to trigger very early, larger TSP instances already exhibit later handovers. On MIS and OP benchmarks, the trigger can occur substantially later, suggesting that the RL solver can remain active for a nontrivial prefix before the conditional DM is invoked.

\begin{table*}[htbp]
\centering
\small
\setlength{\tabcolsep}{4pt}
\renewcommand{\arraystretch}{1.15}
\caption{Average trigger step across datasets. The handover timing varies substantially across problem classes and scales, indicating that HyCO uses an adaptive trigger rather than a fixed early handover.}
\label{tab:trigger_stats}
\begin{tabular}{lccccccccc}
\toprule
& \multicolumn{3}{c}{\textbf{TSP}} 
& \multicolumn{3}{c}{\textbf{MIS}} 
& \multicolumn{3}{c}{\textbf{OP}} \\
\cmidrule(lr){2-4} \cmidrule(lr){5-7} \cmidrule(lr){8-10}
\textbf{Scale / Benchmark}
& 100 & 500 & 1000 
& RB-LARGE & ER-700--800 & SATLIB
& 50 & 100 & 200 \\
\midrule
\textbf{Avg. Trigger Step}
& 0.52 & 6.00 & 8.11
& 56.14 & 10.09 & 81.77
& 3.46 & 1.79 & 20.06 \\
\bottomrule
\end{tabular}
\end{table*}

\newpage
\section{Ablation Study of HyCO}\label{sec:app_ablation}
% \yuheng{Add more ablation experiments done during the ICML rebuttal, and write a summarized paragraph at the beginning. Besides, check through the current table in case there are some errors.}
%This section provides additional ablation studies to analyze the robustness and design choices of HyCO. We first compare adaptive triggering with validation-selected fixed-time schedules, showing that instance-dependent triggering consistently improves over a fixed handover time. We then study the sensitivity of the trigger and decoding hyperparameters, including the DM energy-probe timestep, the RL candidate pool size used for KL estimation, and the candidate selection threshold for DM exploration. Finally, we ablate the trigger signals and show that combining policy entropy with RL--DM disagreement yields better performance than either signal alone.

This section presents additional ablation studies on the robustness and
design choices of HyCO. Section~\ref{sec:app-sensitivity} analyzes
trigger-threshold calibration across TSP-500 and MIS ER-700--800, together
with the sensitivity of other inference-time parameters on TSP-100.
Section~\ref{sec:app-dual-signal} examines the contribution of combining policy
entropy with RL--DM disagreement. Section~\ref{sec:app-dual-track} evaluates
the complementarity of the hybrid path and DM proposal tracks. All TSP ablation experiments are evaluated on 128 test instances, whereas the MIS threshold-sensitivity experiments use a subset of 50 test instances from ER-700–800. %Consequently, the aggregate optimality gaps reported for the MIS sensitivity analysis may differ from those in the main performance tables, which are evaluated on the full test set. The sensitivity results characterize how threshold choices affect performance within this smaller evaluation subset.

\subsection{Hyperparameter Sensitivity of Triggering and DM Exploration}
\label{sec:app-sensitivity}

To assess the robustness of HyCO's inference-time design, we analyze trigger
thresholds across two settings and probe and exploration hyperparameters on
TSP-100. We first evaluate threshold calibration on TSP-500 with LEHD and
Prefix-Difusco, and on MIS ER-700--800 with LWD and Prefix-CoExpander. We then
examine the sensitivity of the DM energy-probe timestep, the RL candidate-pool
size for KL estimation, and the DM exploration budget on TSP-100.

For each setting, we sweep one trigger threshold while holding the other at its
deployed value. Candidate thresholds are obtained by pooling the corresponding
entropy or KL values over calibration trajectories and selecting empirical
quantiles at the 25th, 50th, 75th, 90th, 95th, and 99th percentiles. We also
include the deployed threshold where applicable, together with \emph{always}
and \emph{never} controls. The \emph{never} control disables only the swept
trigger; the other trigger remains active. This scale-aware procedure covers
low-, middle-, and upper-tail signal regimes without requiring arbitrary raw
threshold values.

\begin{table}[htbp]
\centering
\setlength{\belowcaptionskip}{10pt}
\small
\setlength{\tabcolsep}{6pt}
\renewcommand{\arraystretch}{1.2}
\caption{Entropy-threshold $H_{\mathrm{th}}$ sensitivity across TSP-500 with LEHD and
Prefix-Difusco and MIS ER-700--800 with LWD and Prefix-CoExpander. The KL
threshold is fixed at $D_{\mathrm{th}}=30$ for TSP-500 and
$D_{\mathrm{th}}=6.8155$ for MIS. Gap (\%) is lower when better.}
\label{tab:sens-entropy-cross-setting}
\begin{tabular*}{\linewidth}{@{\extracolsep{\fill}}lcccccc}
\toprule
& \multicolumn{3}{c}{\textbf{TSP-500}}
& \multicolumn{3}{c}{\textbf{MIS ER-700--800}} \\
\textbf{Percentile / control}
& \textbf{Threshold}
& \textbf{Gap} $\downarrow$
& \textbf{Avg. step}
& \textbf{Threshold}
& \textbf{Gap} $\downarrow$
& \textbf{Avg. step} \\
\midrule
Always    & Always  & 1.653\% & 1.00 & Always  & 9.98\%  & 0.00 \\
25th      & 0.0011  & 1.653\% & 1.00 & 0.1264  & 9.98\%  & 0.00 \\
50th      & 0.0734  & 1.673\% & 1.03 & 0.7713  & 9.85\%  & 1.38 \\
75th      & 0.5594  & 1.647\% & 1.42 & 1.0002  & 9.36\%  & 23.24 \\
90th      & 0.9480  & 1.600\% & 3.08 & 1.0772  & \textbf{9.23\%} & 36.76 \\
95th      & 1.1569  & 1.598\% & 6.04 & 1.0831  & 9.50\%  & 40.06 \\
%Deployed  & 1.2000  & 1.586 & 5.87 & 0.8000  & 9.81  & 2.30 \\
99th      & 1.5233  & \textbf{1.570\%} & 10.77 & 1.0886 & 10.56\% & 45.15 \\
Never     & Never   & 1.635\% & 2.33 & Never   & 12.52\% & 53.35 \\
\midrule
\multicolumn{7}{l}{\footnotesize The TSP-500 deployed entropy threshold is 1.2 with the gap 1.586\%; the MIS deployed threshold is 0.8 with the gap 9.81\%.} \\
\bottomrule
\end{tabular*}
\end{table}

\begin{table}[htbp]
\centering
\setlength{\belowcaptionskip}{10pt}
\small
\setlength{\tabcolsep}{6pt}
\renewcommand{\arraystretch}{1.2}
\caption{KL-threshold $D_{\mathrm{th}}$ sensitivity across TSP-500 with LEHD and Prefix-Difusco
and MIS ER-700--800 with LWD and Prefix-CoExpander. The entropy threshold is
fixed at $H_{\mathrm{th}}=1.2$ for TSP-500 and $H_{\mathrm{th}}=0.8$ for MIS.
Gap (\%) is lower when better.}
\label{tab:sens-kl-cross-setting}
\begin{tabular*}{\linewidth}{@{\extracolsep{\fill}}lcccccc}
\toprule
& \multicolumn{3}{c}{\textbf{TSP-500}}
& \multicolumn{3}{c}{\textbf{MIS ER-700--800}} \\
\textbf{Percentile / control}
& \textbf{Threshold}
& \textbf{Gap} $\downarrow$
& \textbf{Avg. step}
& \textbf{Threshold}
& \textbf{Gap} $\downarrow$
& \textbf{Avg. step} \\
\midrule
Always    & Always  & 1.653\% & 1.00 & Always  & 9.98\% & 0.00 \\
%7th       & --      & --    & --   & 0.2     & 9.98\% & 0.00 \\
25th      & 0.3047  & 1.653\% & 1.00 & 4.0637  & 9.98\% & 0.00 \\
50th      & 0.8069  & 1.690\% & 1.23 & 5.0974  & 9.98\% & 0.00 \\
75th      & 1.3711  & 1.659\% & 2.24 & 5.4617  & 9.94\% & 0.12 \\
90th      & 2.3090  & 1.663\% & 3.91 & 5.7318  & 9.76\% & 1.12 \\
95th      & 3.4794  & 1.616\% & 4.06 & 5.9009  & \textbf{9.58\%} & 1.68 \\
99th      & 8.4543  & 1.622\% & 4.59 & 6.8155  & 9.81\% & 2.30 \\
99.7th    & 30.0000 & \textbf{1.586\%} & 5.87 & -- & -- & -- \\
Never     & Never   & 1.590\% & 6.02 & Never   & 9.81\% & 2.30 \\
\midrule
\multicolumn{7}{l}{\footnotesize The TSP-500 deployed KL threshold is 30 with the gap 1.586\%; the MIS deployed threshold is 6.8155 with the gap 9.81\%.} \\
\bottomrule
\end{tabular*}
\end{table}

Across both problem families, the raw threshold values differ, but the quantile sweeps reveal a consistent calibration rule. Low- and mid-quantile thresholds tend to trigger diffusion prematurely, whereas effective configurations lie in a \textbf{high but non-extreme} signal region. On TSP-500, entropy thresholds from the 90th to the 99th
percentile yield gaps between 1.570\% and 1.600\%; KL thresholds at the 95th
percentile and above yield gaps between 1.586\% and 1.622\%. On MIS
ER-700--800, the best results occur at the 90th percentile for entropy and the
95th percentile for KL. Raising the entropy threshold to the 99th percentile
worsens the gap, suggesting that an overly restrictive threshold can delay
useful DM invocation. These results support using empirical quantiles to
calibrate setting-specific thresholds rather than transferring raw threshold
values unchanged across problem families and backbone pairs.

Beyond threshold calibration, we study other inference-time parameters on
TSP-100. These experiments examine the sensitivity of HyCO to the DM
energy-probe timestep, the RL candidate-pool size used for KL estimation, and
the DM exploration budget. They characterize how these choices affect solution
quality and inference cost.

%To empirically validate the theoretical framework established in Section 4, we conduct extensive ablation studies on the TSP-100 benchmark. These experiments aim to bridge the gap between our expectation-level theory and trajectory-level inference by scrutinizing how HyCO’s core components navigate the regime shift between sequential construction and global generation. We specifically focus on the effectiveness of the dual-signal trigger and the structural robustness of the framework across varying hyperparameter settings.

Table \ref{tab:ablation_probe_t} examines the sensitivity of our framework to the DM energy probe's timestep (\texttt{dm\_probe\_timestep}). The results demonstrate remarkable robustness: across a wide range of timesteps from 100 to 900, the optimality gap remains stable at 0.02\%. This is a significant practical advantage, as it indicates that the single-step energy probe is a reliable signal that does not require meticulous hyperparameter tuning.

\begin{table}[htbp!]
\centering
%\scriptsize
%\small
\setlength{\belowcaptionskip}{10pt}  % 控制 caption 到后面内容的距离

    \centering
    \setlength{\tabcolsep}{5pt}
    % CORRECTED: Escaped underscore in \texttt
    \caption{Sensitivity analysis on the DM energy probe timestep, \texttt{dm\_probe\_timestep}.}
    \label{tab:ablation_probe_t}
    \begin{tabular}{cccc}
        \toprule
        % CORRECTED: Removed trailing ~
        \textbf{Energy Probe Ts} & \textbf{Gap (\%)} $\downarrow$ & \textbf{Time (s)} & \textbf{Ave. trigger step (Trig. Rates)} \\
        \midrule
        100 & 0.02 & 1.762s & 0.66 (100.0\%)\\
        300 & 0.02 & 1.768s & 0.66 (100.0\%)\\
        \textbf{500} & 0.02 & \textbf{1.749s} & 0.52 (100.0\%)\\
        700 & 0.02 & 1.755s & 0.51 (100.0\%)\\
        900 & 0.02 & 1.753s & 0.51 (100.0\%)\\
        \bottomrule
    \end{tabular}
\end{table}

Table \ref{tab:ablation_top_m} investigates the impact of the RL candidate pool size (\texttt{probe\_rl\_top\_m}) used for the KL divergence calculation. This experiment highlights the importance of providing a sufficiently large set of candidate actions for the probe. With too few candidates (e.g., 5), the KL divergence is not a reliable indicator, resulting in a poor optimality gap (1.30\%) and a low trigger rate (43.0\%). Our default setting of 15 ensures that the divergence metric is calculated over a meaningful distribution, allowing for effective detection of critical junctures.

\begin{table}[htbp!]
\centering
%\scriptsize
%\small
\setlength{\belowcaptionskip}{10pt}  % 控制 caption 到后面内容的距离
    \centering
    \setlength{\tabcolsep}{5pt}
    % CORRECTED: Escaped underscore in \texttt
    \caption{Sensitivity analysis on the RL candidate pool size, \texttt{probe\_rl\_top\_m}, used for KL divergence calculation.}
    \label{tab:ablation_top_m}
    \begin{tabular}{cccc}
        \toprule
        % CORRECTED: Escaped underscore in \texttt
        \textbf{\texttt{Probe\_RL\_Top\_M}} & \textbf{Gap (\%)} $\downarrow$ & \textbf{Time (s)} & \textbf{Ave. trigger step (Trig. Rates)} \\
        \midrule
        5 & 1.30 & 0.761s & 4.05 (43.0\%) \\
        10 & 0.15 & 1.278s & 5.20 (97.1\%) \\
        15 & 0.02 & 1.749s & 0.52 (100.0\%) \\
        20 & 0.01 & 2.364s & 0.00 (100.0\%) \\
        \bottomrule
    \end{tabular}
\end{table}

Table \ref{tab:ablation_probe_m} analyzes the effect of the candidate selection threshold for DM exploration (\texttt{TopN\_cum\_Th}). This parameter controls the breadth of the DM's search once it is triggered. The results show a clear trade-off: a small threshold (e.g., 0.2) is faster but often fails to find a high-quality solution, yielding a suboptimal 0.12\% gap. Increasing the exploration breadth is crucial for capitalizing on the DM's generative power. Our default value of 0.8 allows the DM to explore a diverse set of high-probability candidates, which is vital for discovering the near-optimal paths that lead to our state-of-the-art 0.02\% gap.

\begin{table}[htbp!]
\centering
%\scriptsize
%\small
\setlength{\belowcaptionskip}{10pt}  % 控制 caption 到后面内容的距离
    \centering
    \setlength{\tabcolsep}{5pt}
    % CORRECTED: Escaped underscore in \texttt
    \caption{Ablation study on the candidate selection strategy for DM exploration, \texttt{n\_cumulative\_threshold}.}
    \label{tab:ablation_probe_m}
    \begin{tabular}{cccc}
        \toprule
        % CORRECTED: Escaped underscore in \texttt
        \textbf{\texttt{TopN\_cum\_Th}} & \textbf{Gap (\%)} $\downarrow$ & \textbf{Time (s)} & \textbf{Ave. trigger step (Trig. Rates)} \\
        \midrule
        0.2 & 0.12 & 0.158s & 0.52 (100.0\%) \\
        0.4 & 0.12 & 0.570s & 0.52 (100.0\%) \\
        0.5 & 0.08 & 0.723s & 0.52 (100.0\%) \\
        0.6 & 0.05 & 0.942s & 0.52 (100.0\%) \\
        0.8 & 0.02 & 1.749s & 0.52 (100.0\%) \\
        \bottomrule
    \end{tabular}
\end{table}

\subsection{Necessity of Dual-Signal Trigger Design}
\label{sec:app-dual-signal}

Table \ref{tab:ablation_core_components} provides a comprehensive analysis of our core trigger mechanism. The results clearly show that the dual-criteria trigger, which combines both RL policy entropy and KL divergence, is essential for top performance. Relying solely on the KL-divergence or entropy trigger leads to significantly worse optimality gaps (0.09\% and 0.05\%, respectively). This confirms our hypothesis that RL policy uncertainty (entropy) and model disagreement (KL) are complementary signals. The former identifies when the RL agent is indecisive, while the latter detects when it is confidently wrong. Together, they form a robust and effective condition for invoking the diffusion model.

\begin{table}[htbp]\label{tab:trigger}
\centering
\setlength{\belowcaptionskip}{10pt}  % 控制 caption 到后面内容的距离
    \centering
    \small
    \setlength{\tabcolsep}{8pt} % Adjust column spacing
    \renewcommand{\arraystretch}{1.2} % Adjust row height
    \caption{Complete Ablation study on the core components of HyCO, evaluated on the TSP-100 dataset. Gap (\%) is the optimality gap compared to the Concorde solver. Time (s) is the average inference time per instance. Avg. Trigger Step Index indicates the average step number in the tour construction at which the Diffusion Model was first invoked. Trigger Rate denotes the percentage of instances where the DM was triggered at least once. The best-performing model is highlighted in bold.}
    \label{tab:ablation_core_components}
    \begin{tabular}{llccc}
        \toprule
        \textbf{Model / Method} & \textbf{Description: Triggers on} & \textbf{Gap (\%)} $\downarrow$ & \textbf{Time (s)} & \textbf{Ave. trigger step} \\
         &  & &  & \textbf{(Trig. Rates)} \\
        \midrule
        \multicolumn{5}{l}{\textit{Trigger Mechanism Ablation }} \\
        HyCO (Entropy) &  policy entropy only & 0.05 & 1.903s & 2.36 (99.8\%) \\
        HyCO (KL) & KL divergence only & 0.09 & 1.305s & 1.35 (97.0\%) \\
        \midrule
        \textbf{HyCO (Full Model)} & \textbf{Full Model (Entropy + KL)} & \textbf{0.02} & 1.749s & 0.52 (100.0\%) \\
        \bottomrule
    \end{tabular}
\end{table}

\subsection{Complementarity of the Two Output Tracks}
\label{sec:app-dual-track}

To isolate the effect of output-track selection, we evaluate both tracks on the
same 128 TSP-500 instances while keeping the adaptive trigger and DM candidate
pool fixed. We change only which output HyCO returns. ``Hybrid path only''
returns the corrected RL continuation; ``DM proposal only'' returns the best
conditional-DM proposal; and full HyCO returns the better solution from the two
tracks. Table~\ref{tab:dual_track_ablation} reports the optimality gaps and the
track selected by full HyCO for each instance.

\begin{table}[htbp]
\centering
\setlength{\belowcaptionskip}{10pt}
\small
\setlength{\tabcolsep}{8pt}
\renewcommand{\arraystretch}{1.2}
\caption{Ablation of the two output tracks on TSP-500. Gap (\%) is the
optimality gap to the Concorde solver; lower is better. Track-selection counts
are out of 128 instances.}
\label{tab:dual_track_ablation}
\begin{tabular}{lccc}
    \toprule
    \textbf{Backbone pair}
    & \textbf{Hybrid path only}
    & \textbf{DM proposal only}
    & \textbf{Full HyCO} \\
    &
    \textbf{Gap (\%)} $\downarrow$
    & \textbf{Gap (\%)} $\downarrow$
    & \textbf{Gap (\%)} $\downarrow$ \\
    \midrule
    POMO + Prefix-Difusco
    & 13.78\% & 3.65\% & \textbf{3.56\%} \\
    LEHD + Prefix-Difusco
    & 1.61\% & 3.23\% & \textbf{1.57\%} \\
    \bottomrule
\end{tabular}
\end{table}

We also counted which track supplied the solution selected by full HyCO on each
instance. This reveals a clear backbone-dependent pattern: with POMO, the DM
proposal track supplies the selected solution on 127 of 128 instances, whereas
with LEHD, the hybrid path does so on 124 of 128 instances. These results show that \textbf{the two tracks are complementary rather than uniformly beneficial}. With the weaker POMO backbone, the DM proposal track is selected on 127/128 instances and achieves a 3.65\% gap, while the hybrid path alone has a 13.78\% gap. With the stronger LEHD backbone, the pattern reverses: the hybrid path is selected on 124/128 instances and achieves a 1.61\% gap, whereas the DM proposal track alone has a 3.23\% gap. Full HyCO obtains the best result in both settings (3.56\% and 1.57\%, respectively). Thus, the dual-track design provides robustness to the relative strengths of the RL prefix and conditional DM completion, rather than relying on one track to dominate in all settings. 

\subsection{Effect of Multiple Triggering on Solution Quality and Inference Time}
\label{app:single-vs-multiple-trigger}

The single-handover design is motivated by the theoretical result, which identifies a one-switch optimum under the complementary error-profile assumptions. Because learned backbones may depart from these idealized profiles, we additionally evaluated an iterative multi-trigger variant using POMO as the RL backbone and Prefix-Difusco as the DM backbone. Whenever the trigger is activated again after RL continuation, the multi-trigger variant re-runs the full dual-track procedure. All data are shown in Table~\ref{tab:single-vs-multiple-trigger}.

\begin{table}[htbp]
\centering
\setlength{\belowcaptionskip}{10pt}
\small
\setlength{\tabcolsep}{8pt}
\renewcommand{\arraystretch}{1.2}
\caption{Single-trigger and multiple-trigger HyCO with POMO and Prefix-Difusco.
Gap (\%) is the optimality gap to Concorde; lower is better. Time is the
average inference time per instance.}
\label{tab:single-vs-multiple-trigger}
\begin{tabular}{lcccc}
\toprule
\textbf{Dataset}
& \multicolumn{2}{c}{\textbf{Single trigger}}
& \multicolumn{2}{c}{\textbf{Multiple triggers}} \\
& \textbf{Gap} $\downarrow$
& \textbf{Time}
& \textbf{Gap} $\downarrow$
& \textbf{Time} \\
\midrule
TSP-50   & 0.03\% & 0.25\,s    & 0.02\% & 2.38\,s \\
TSP-100  & 0.02\% & 1.20\,s    & 0.01\% & 23.28\,s \\
TSP-500  & 3.56\% & 1.20\,min  & 1.37\% & 6.43\,min \\
TSP-1000 & 4.58\% & 2.43\,min  & 4.47\% & 14.22\,min \\
\bottomrule
\end{tabular}
\end{table}

The results show that multiple triggering consistently improves solution quality, but the gain is modest in most settings relative to its computational cost. On TSP-50, TSP-100, and TSP-1000, it improves the gap by only 0.01\%, 0.01\%, and 0.11\%, respectively, while increasing inference time by 9.5$\times$, 19.4$\times$, and 5.9$\times$. On TSP-500, it yields a larger improvement, reducing the gap from 3.56\% to 1.37\%, but requires 5.4$\times$ more inference time, increasing from 1.20 to 6.43 minutes. Accordingly, the single-handover configuration is our default efficiency-oriented setting.

\newpage
\section{Implementation Details of Prefix\_Difusco and RL models}
\label{sec:app_implement_details}

Here we provide detailed specifications for our backbone conditional diffusion model, Prefix\_Difusco, used in the experiments. About the RL model, we uniformly use the code and environment provided by the RL4CO \cite{berto2023rl4co} package for training \footnote{https://github.com/ai4co/rl4co}. The detailed hyperparameters are listed in \ref{hyper_tab}.

\subsection{Model Architecture}
The core of Prefix\_Difusco modifies the GNN architecture from DIFUSCO\cite{Sun2023} for the conditional generation task. The key components are:
\begin{itemize}
    \item \textbf{Node Feature Embedding}: Node coordinates are first converted into high-dimensional features using a sinusoidal positional embedding. A binary feature indicating whether a node is part of the prefix is concatenated to this embedding. A final linear layer projects this combined feature vector to the GNN's expected input dimension ($d_{node} = 128$).
    \item \textbf{PrefixEncoder}: An LSTM-based encoder takes the sequence of node features corresponding to the prefix tour and outputs a single global conditioning vector ($d_{cond} = 256$). This vector summarizes the properties of the given partial tour.
    \item \textbf{DifuscoGNNEncoder}: This is the main denoising network. It is a 12-layer GNN that processes a graph where nodes have the features described above. At each layer, the GNN's message passing is conditioned by both the global prefix vector from the \texttt{Prefix\_Encoder} and a sinusoidal embedding of the current timestep $t$.
    \item \textbf{Output Head}: The GNN outputs a logit for each potential edge in the graph, representing the probability of that edge being part of the optimal tour.
\end{itemize}

\subsection{Training Procedure}
The model is trained to predict the ground-truth adjacency matrix $x_0$ from a noised version $x_t$ and a conditional prefix.
\begin{itemize}
    \item \textbf{Dataset and Conditioning}: The training dataset consists of TSP instances and their optimal tours. For each sample, we derive a training instance by randomly selecting a prefix of length $k$ from the optimal tour.
    \item \textbf{Diffusion Process}: We use a discrete diffusion process over $T=1000$ steps with a cosine noise schedule to corrupt the ground-truth adjacency matrix $x_0$ into a noisy matrix $x_t$.
    \item \textbf{Masked Loss Function}: The model's objective is to minimize the Binary Cross-Entropy (BCE) loss between the predicted adjacency matrix and the ground truth $x_0$. Crucially, the loss is only computed on the ``suffix'' edges—that is, all edges except those whose both endpoints are within the given prefix. This forces the model to learn how to best complete the tour.
    \item \textbf{Curriculum Learning}: To improve convergence and performance, we employ a multi-stage curriculum. The training starts with a distribution of long prefixes (e.g., $k \in [60,90]$), making the completion task easier. As training progresses, the distribution of $k$ shifts towards shorter, more difficult prefixes (e.g., $k \in [1,30]$), allowing the model to gradually master the full conditional generation task.
\end{itemize}

\subsubsection{Training for TSP-50}
The \texttt{Prefix\_Difusco} model for TSP-50 was trained from scratch. We employed a 5-stage curriculum learning strategy designed to gradually increase the task difficulty. Each stage was trained for 10 epochs, initializing from the best checkpoint of the previous stage.
\begin{itemize}
    \item \textbf{Stage 1 (Easy):} Trained on long prefixes with lengths $k \in [30, 49]$.
    \item \textbf{Stage 2 (Medium):} Trained on prefix lengths $k \in [10, 30]$.
    \item \textbf{Stage 3 (Hard):} Trained on the full range of prefix lengths $k \in [1, 49]$.
    \item \textbf{Stage 4 (Short Focus):} Focused on short prefixes with lengths $k \in [1, 20]$.
    \item \textbf{Stage 5 (Very Short Focus):} Further focused on very short prefixes with $k \in [1, 10]$ to enhance performance on early-step decisions.
\end{itemize}

\subsubsection{Training for TSP-100}
The model for TSP-100 was also trained from scratch following a similar multi-stage curriculum learning approach. The prefix length ranges for each stage were adjusted proportionally for the larger problem size to ensure a smooth learning progression from easy to hard completion tasks. The core hyperparameters, such as hidden dimensions and learning rate, were kept consistent with the TSP-50 model.
\begin{itemize}
    \item \textbf{Stage 1 (Easy):} Trained on long prefixes with lengths $k \in [61, 99]$.
    \item \textbf{Stage 2 (Medium):} Trained on prefix lengths $k \in [30, 60]$.
    \item \textbf{Stage 3 (Hard):} Trained on the full range of prefix lengths $k \in [1, 99]$.
    \item \textbf{Stage 4 (Short Focus):} Focused on short prefixes with lengths $k \in [1, 20]$.
    \item \textbf{Stage 5 (Very Short Focus):} Further focused on very short prefixes with $k \in [1, 10]$ to enhance performance on early-step decisions.
\end{itemize}

\subsubsection{Training for TSP-500}
Due to the significantly larger scale of TSP-500, we adopted a more advanced training strategy combining \textbf{transfer learning} and a tailored curriculum.
\begin{itemize}
    \item \textbf{Transfer Learning}: The TSP-500 model was not trained from scratch. Instead, it was initialized using the weights from our best-trained TSP-100 checkpoint. Weights were transferred for all layers with matching names and shapes (e.g., GNN layers, prefix encoder), providing a strong starting point and accelerating convergence.
    \item \textbf{Curriculum Learning}: After initialization, the model was fine-tuned on TSP-500 data using a 5-stage curriculum similar to the one for TSP-50, but with ranges adjusted for $N=500$ (e.g., Stage 1: $k \in [50, 100]$, Stage 2: $k \in [20, 50]$, etc.). For the results reported in this paper, we ran an accelerated training schedule of approximately 20 epochs in total, focusing on the most critical curriculum stages to balance performance and computational cost.
\end{itemize}

\subsubsection{Training for TSP-1000}
For the largest scale, TSP-1000, we continued the strategy of combining transfer learning with a specialized curriculum to manage the increased complexity and computational demands.
\begin{itemize}
    \item \textbf{Transfer Learning}: The TSP-1000 model was initialized with the weights from our best-performing TSP-100 checkpoint. This transfer learning approach provided a robust feature foundation, significantly accelerating the training convergence on the larger graph size.
    \item \textbf{Curriculum Learning}: Following initialization, the model was fine-tuned on the TSP-1000 dataset using a 4-stage curriculum over a total of 25 epochs. The training began with easier tasks (completing tours from long prefixes, with k up to 500) and progressively moved to more difficult scenarios, focusing on shorter prefixes (k down to 1) in later stages to refine the model's ability to make critical early decisions.
\end{itemize}

\subsection{Decoding Algorithms}
To convert the probabilistic heatmap output from our \texttt{\textbf{Prefix\_Difusco}} model into a valid TSP tour, we employ a deterministic greedy decoding strategy inspired by DIFUSCO \citep{Sun2023}. This approach ensures that given a heatmap, the resulting tour is always the same, which is crucial for the stability of the HyCO framework. The core decoding process follows a principled, multi-stage procedure designed to construct high-quality tours while respecting the prefix constraints.

The decoding algorithm proceeds as follows:
\begin{itemize}
    \item \textbf{Edge Score Calculation:} The raw adjacency probability matrix $P$ from the diffusion model is first symmetrized to ensure consistency ($P' = (P + P^T) / 2$). To favor shorter edges, which are fundamental to good TSP solutions, we compute an edge score for each potential edge $(i, j)$ by dividing its symmetrized probability by its Euclidean distance: $S_{ij} = P'_{ij} / \text{dist}(i, j)$. All possible edges are then sorted in descending order based on these scores.

    \item \textbf{Enforce Prefix Constraint:} Before any greedy selection, the decoder first enforces the given conditional prefix. All edges that form the given partial tour are mandatorily included in the solution set. A Union-Find data structure is initialized, and the degrees of the prefix nodes are updated accordingly to ensure these edges are fixed.

    \item \textbf{Greedy Spanning Path Construction:} The algorithm iterates through the globally sorted list of edges. For each candidate edge, it performs three checks:
    \begin{enumerate}
        \item It is not an existing prefix edge.
        \item Adding the edge will not result in any node having a degree greater than two.
        \item Adding the edge will not form a premature cycle (verified using the Union-Find data structure).
    \end{enumerate}
    If all conditions are met, the edge is added to the solution set, and the node degrees and Union-Find structure are updated. This process continues until a total of $N-1$ edges have been selected, forming a spanning path of all nodes.

    \item \textbf{Tour Finalization:} Once a spanning path of $N-1$ edges is formed, there will be exactly two nodes with a degree of one (the endpoints of the path). The final edge connecting these two endpoints is deterministically added to close the path and form a valid Hamiltonian cycle. The final list of $N$ edges is then converted into a sequential tour starting from the first node of the original prefix (or node 0 if no prefix was given).
\end{itemize}

\subsection{Hyperparameters}\label{hyper_tab}
This subsection provides a comprehensive summary of the hyperparameters for the core models employed in the HyCO framework: The configurations for our conditional diffusion model (\texttt{Prefix\_Difusco}), the Attention Model (AM), and POMO \cite{Kool2019,Kwon2020} are detailed in Table \ref{tab:hyperparams_all}, Table \ref{tab:hyperparams_all_2}, and Table \ref{tab:hyperparams_all_3}, respectively. To ensure rigorous comparability, we strictly adopted the architectural configurations (e.g., layers, dimensions) from the original DIFUSCO and RL4CO implementations, thereby isolating the performance gains to the HyCO framework logic rather than backbone scaling. Besides, for inference-specific parameters such as DDIM sampling steps, values were selected based on a preliminary analysis of the speed-quality trade-off.

\begin{table}[htbp]
\centering
%\scriptsize
\small
\setlength{\belowcaptionskip}{10pt}  % Controls the distance from the caption to the content below
\centering
\caption{Hyperparameters for Prefix\_Difusco across different problem sizes.}
\label{tab:hyperparams_all}
\setlength{\tabcolsep}{10pt} % Adjust column spacing
\renewcommand{\arraystretch}{1.2} % Controls row height
\begin{tabular}{@{}lcccc@{}}
\toprule
\textbf{Parameter} & \textbf{TSP-50} & \textbf{TSP-100} & \textbf{TSP-500} & \textbf{TSP-1000} \\
\midrule
\multicolumn{5}{l}{\textit{\textbf{Model Architecture}}} \\
Node Count (N) & 50 & 100 & 500 & 1000 \\
sparse factor(K) & N.A & N.A & N.A & 100 \\
GNN Layers (L) & 12 & 12 & 12 & 12 \\
Hidden Dimension & 256 & 256 & 256 & 256 \\
Node Embedding Dim & 128 & 128 & 128 & 128 \\
Prefix Condition Dim & 256 & 256 & 256 & 256 \\
\midrule
\multicolumn{5}{l}{\textit{\textbf{Diffusion Process}}} \\
Timesteps (T) & 1000 & 1000 & 1000 & 1000 \\
Beta Schedule & cosine & cosine & cosine & cosine \\
Inference Steps & 10 & 10 & 50 & 50 \\
  Inference Sampler &  DDIM  &  DDIM &  DDIM & DDIM \\
\midrule
\multicolumn{5}{l}{\textit{\textbf{Training}}} \\
Batch Size (per GPU) & 128 & 96 & 4 & 8 \\
Epoch & 50 & 50 & 20 & 25 \\
Training data (per epoch) & 1500000 & 1500000 & 128000 & 65000 \\
Learning Rate & 2e-4 & 2e-4 & 2e-5 & 1e-4 \\
Optimizer & Adam & Adam & Adam & Adam \\
Training Method & Curriculum & Curriculum & Transfer\&Curriculum & Transfer\&Curriculum \\
Environment & \multicolumn{4}{c}{2x NVIDIA A40 GPUs} \\
\bottomrule
\end{tabular}
\end{table}

\begin{table}[p]
\centering
%\scriptsize
%\small
\setlength{\belowcaptionskip}{10pt}  % 控制 caption 到后面内容的距离
\centering
\caption{Hyperparameters for Attention Model across different problem sizes. Need to notice, AM and POMO models for TSP-500 are trained on the tsp-200 instances and then generalized into the TSP 500.}
\label{tab:hyperparams_all_2}
\setlength{\tabcolsep}{20pt} % Adjust column spacing
\begin{tabular}{@{}lccc@{}}
\toprule
\textbf{Parameter} & \textbf{TSP-50} & \textbf{TSP-100} & \textbf{TSP-500} \\
\midrule
\multicolumn{4}{l}{\textit{\textbf{Model Architecture}}} \\
GNN Layers (L) & 3 & 3 & 3 \\
Hidden Dimension & 512 & 512 & 512 \\
Node Embedding Dim & 128 & 128 & 128 \\
Attention heads & 8 & 8 & 8 \\
Baseline & rollout & rollout & critic \\
\midrule
\multicolumn{4}{l}{\textit{\textbf{Training}}} \\
Batch Size (per GPU) & 512 & 512 & 1024 \\
Training data (each epoch) & 1280000 & 1280000 & 1280000 \\
Epoch & 100 & 100 & 120 \\
Learning Rate & 1e-4 & 1e-4 & 1e-4 \\
LR Scheduler & \multicolumn{3}{c}{multistep LR, gamma=0.1, milestone = [80, 95]} \\
Normalization & \multicolumn{3}{c}{batch} \\
Environment & \multicolumn{3}{c}{1x NVIDIA A40 GPUs} \\
\bottomrule
\end{tabular}
\end{table}

\begin{table}[htbp!]
\centering
%%\scriptsize
%\small
\setlength{\belowcaptionskip}{10pt}  % 控制 caption 到后面内容的距离

\centering
\caption{Hyperparameters for POMO across different problem sizes.}
\label{tab:hyperparams_all_3}
\setlength{\tabcolsep}{20pt} % Adjust column spacing
\begin{tabular}{@{}lccc@{}}
\toprule
\textbf{Parameter} & \textbf{TSP-50} & \textbf{TSP-100} & \textbf{TSP-500} \\
\midrule
\multicolumn{4}{l}{\textit{\textbf{Model Architecture}}} \\
GNN Layers (L) & 6 & 6 & 6 \\
Hidden Dimension & 512 & 512 & 512 \\
Node Embedding Dim & 128 & 128 & 128 \\
Attention heads & 8 & 8 & 8 \\
Augment & 8 & 8 & 8 \\
\midrule
\multicolumn{4}{l}{\textit{\textbf{Training}}} \\
Batch Size (per GPU) & 512 & 512 & 256 \\
Training data (each epoch) & 100000 & 100000 & 100000 \\
Epoch & 400 & 800 & 400 \\
Learning Rate & 1e-4 & 1e-4 & 1e-4 \\
LR Scheduler & \multicolumn{3}{c}{multistep LR, gamma=0.1, milestone = [80, 95]} \\
Normalization & \multicolumn{3}{c}{instance} \\
Environment & \multicolumn{2}{c}{1x NVIDIA A40 GPUs} & 2x NVIDIA A40 GPUs \\
\bottomrule
\end{tabular}
\end{table}

\section{Limitations and Broader Impacts}
\label{sec:limitations_impacts}

HyCO's theoretical analysis relies on explicit error-scaling and boundary assumptions on the RL and DM backbones. Extending the framework to other solver combinations may require re-examining these assumptions and adapting the trigger signals accordingly.

HyCO is a general neural combinatorial optimization method and is not tied to a specific high-stakes deployment. Its potential positive impact lies in improving the efficiency and quality of optimization systems, with possible applications in routing, scheduling, logistics, and resource allocation. Potential negative impacts could arise if such tools are deployed in sensitive domains where suboptimal objectives or incorrect solutions affect access to resources or services. In such settings, deployment should include domain-specific constraints, validation, and human oversight.

%%%%%%%%%%%%%%%%%%%%%%%%%%%%%%%%%%%%%%%%%%%%%%%%%%%%%%%%%%%%

\newpage

\end{document}